\documentclass{article}
\usepackage{cite}
\usepackage{amsmath,amssymb,amsfonts}
\usepackage{algorithmic}
\usepackage{graphicx}
\usepackage{algorithm,algorithmic}
\usepackage{hyperref}
\hypersetup{hidelinks=true}
\usepackage{textcomp}

\usepackage{tabularx}

\newtheorem{theorem}{Theorem}

\newtheorem{lemma}{Lemma}
\newtheorem{proof}{proof}

\newtheorem{definition}{Definition}

\def\BibTeX{{\rm B\kern-.05em{\sc i\kern-.025em b}\kern-.08em
    T\kern-.1667em\lower.7ex\hbox{E}\kern-.125emX}}

\begin{document}
\title{Homomorphic Mappings for Value-Preserving State Aggregation in Markov Decision Processes}

\author{ Shuo~Zhao \thanks{Zhejiang University of Technology, China. 
\texttt{2112103033@zjut.edu.cn}},
Yongqiang~Li \thanks{Zhejiang University of Technology, China. 
\texttt{yqli@zjut.edu.cn}},
Yu~Feng\thanks{Zhejiang University of Technology, China. 
\texttt{yfeng@zjut.edu.cn}}
Zhongsheng~Hou\thanks{Qingdao University, China. 
\texttt{zhshhou@bjtu.edu.cn}},
Yuanjing~Feng\thanks{Zhejiang University of Technology, China. 
\texttt{fyjing@zjut.edu.cn}},
\thanks{This work has been submitted to the IEEE for possible publication. 
Copyright may be transferred without notice, after which this version may no longer be accessible.}
}

\maketitle

\begin{abstract}
  State aggregation aims to reduce the computational complexity of solving Markov Decision Processes (MDPs) while preserving the performance of the original system. A fundamental challenge lies in optimizing policies within the aggregated, or abstract, space such that the performance remains optimal in the ground MDP-a property referred to as \texttt{"}optimal policy equivalence\texttt{"}.
  This paper presents an abstraction framework based on the notion of homomorphism, in which two Markov chains are deemed homomorphic if their value functions exhibit a linear relationship. Within this theoretical framework, we establish a sufficient condition for the equivalence of optimal policy.
  We further examine scenarios where the sufficient condition is not met and derive an upper bound on the approximation error and a performance lower bound for the objective function under the ground MDP. We propose Homomorphic Policy Gradient (HPG), which guarantees optimal policy equivalence under sufficient conditions, and its extension, Error-Bounded HPG (EBHPG), which balances computational efficiency and the performance loss induced by aggregation. In the experiments, we validated the theoretical results and conducted comparative evaluations against seven algorithms.
\end{abstract}

\section{Introduction}

As Markov Decision Processes (MDPs) are increasingly applied to complex real-world problems, understanding their structure and applications in reinforcement learning becomes ever more important \cite{silver2016mastering,vinyals2019grandmaster,10463092}.  However, the computational complexity of solving large-scale MDPs remains a significant challenge due to the exponential growth of the state space \cite{gheshlaghi2013minimax,9022871,sidford2023variance}. State aggregation has long been considered a key strategy for addressing this issue by compressing the state space while retaining relevant decision-making properties \cite{1101705,1103941,9347729,10535739}. The core objective of this study is to ensure that optimal policy in the aggregated, or abstract, space remain optimal in the ground MDP-a property we refer to as \textbf{optimal policy equivalence}. 


State aggregation reduces the computational complexity of planning and learning by grouping similar states into abstract classes, which aim to preserve the essential structure of the decision process. This paradigm has found applications in multi-agent coordination \cite{ZHANG2025130430}, visual representation learning \cite{9082110}, and operational systems \cite{4429842,8691525}. Existing state abstraction methods can broadly be classified into two categories: feature-based (or structural) and value-based aggregation.

Early efforts in state abstraction often rely on feature-based representations. These methods employ hand-crafted or learned feature functions to map raw states into a lower-dimensional space, where aggregation can be performed more effectively \cite{NIPS1994_287e03db, van2006performance, pmlr-v97-yang19b, 6804680}. For instance, Guestrin et al. leverage dynamic Bayesian networks to encode structured state features \cite{guestrin2003efficient} and Zhang et al. investigate spectral properties of Markov chains to assess the feasibility of aggregation via rank-based analysis \cite{8918022}. A related line of work explores matrix factorization techniques, such as perturbation analysis \cite{1102780,67293} and soft clustering \cite{NIPS1994_287e03db}, to enable compact representations. While these approaches can yield informative abstractions, they often require significant computational resources, particularly in high-dimensional settings.

Compared to feature-based methods, value-based aggregation focuses on minimizing value function approximation error and makes it more suitable for studying the relationship between the value functions of the Markov chains before and after aggregation. These methods typically construct abstractions that allow for approximate policy evaluation and improvement with provable guarantees. Adaptive iterative aggregation algorithms \cite{bertsekas1996neuro,24227,pmlr-v108-abel20a,pmlr-v119-ayoub20a} exemplify this idea by iteratively refining the aggregation scheme to minimize error in value prediction. Theoretical results further reveal that the number of abstract states grows polynomially with the complexity of the optimal value function \cite{5483078}.
Recent advances have extended this perspective to sample-efficient reinforcement learning. Notably, Abel et al. establish a quantitative relationship between Q-function complexity and the granularity of the required aggregation in lifelong learning settings.  Approximate aggregation techniques are also shown to offer superior generalization, especially in model-free environments \cite{pmlr-v48-abel16, Abel_2019_a,5483078,ishfaq21a,Forghieri_Castel_Hyon_Pennec_2024,DBLP}. To support abstraction without prior knowledge of the MDP, adaptive value iteration algorithms have been proposed \cite{chen2022adaptive,pmlr-v216-geng23a}.

However, most of the aforementioned methods lack theoretical tools for analyzing the optimal policy equivalence , especially in the context of automated or learned abstractions. A promising theoretical framework is the homomorphic MDP theory proposed by Ravindran \cite{ravindran2004approximate}, which formalizes policy-preserving abstractions by defining structure-preserving mappings between MDPs. Closely related is the notion of bisimulation \cite{larsen1991bisimulation, givan2003equivalence}, a classical concept of behavioral equivalence that has been extended to MDPs and used to define state aggregation schemes that guarantee the preservation of value functions and optimal policy \cite{ferns2005metrics, ferns2011bisimulation}. Shoshtari et al. further demonstrated that homomorphic MDPs ensure optimal policy equivalence under such abstractions \cite{ferns2014bisimulation}. Building upon this, Ferns et al. extended bisimulation metrics to continuous state spaces, showing that policy equivalence can still be guaranteed under appropriate metric conditions \cite{NEURIPS2022_7f44f98e,JMLR:v25:23-1415}. Despite their strong theoretical guarantees, these frameworks require that the abstract MDP exactly preserve both the reward and transition dynamics of the original MDP-an assumption that is often too restrictive for practical applications.

In this work, we first draw an analogy to homomorphic MDPs and propose a framework of homomorphic Markov chains and homomorphic mappings. Within this framework, we derive a sufficient condition for the equivalence of optimal policy, which is strictly weaker than the corresponding condition required by homomorphic MDPs.
Building on this theoretical foundation, we propose two practical algorithms. We first introduce the Homomorphic Policy Gradient (HPG) method, which guarantees optimal policy equivalence, ensuring that the performance of optimal policy is equivalent to that of the original problem.
When exact value preservation is infeasible, we perform a least-squares projection to relax the constraints and derive a provable upper bound on the induced error. Based on this result, we develop the Error-Bounded Homomorphic Policy Gradient (EBHPG) algorithm, which achieves a favorable trade-off between computational efficiency and performance degradation.
We evaluate our approach across synthetic and structured environments, including weakly coupled MDPs, FourRooms navigation, and queuing networks. Results demonstrate improved training efficiency and competitive policy quality relative to classical aggregation techniques.

The remainder of this paper is organized as follows.
In Section \ref{sec_1}, we introduced the fundamental notation for MDPs and presented the concept of homomorphic MDPs.  In the Section \ref{sec_2}, we established that homomorphic mappings induce a linear relationship between value functions and provided sufficient conditions for optimal policy equivalence. In the Section \ref{sec_3}, we analyze policy optimization under both exact and inexact optimal policy equivalence, deriving feasible descent directions and bounding the approximation error in the latter case. In Section \ref{sec_4}, we empirically validate our methods on benchmark tasks, demonstrating their effectiveness and robustness under both exact and approximate homomorphism settings. In Section \ref{sec_5}, we summarize the contributions of this work and analyze the limitations of existing approaches.
We use uppercase letters (e.g., $S_t$) to denote random variables, and lowercase letters (e.g., $s_t$) to denote their realizations. The cardinality of a set $\mathcal{S}$ is denoted by $|\mathcal{S}|$. For any matrix $A$, we denote its inverse by $A^{-1}$, its transpose by $A^\top$, and its trace by $\mathrm{Tr}(A)$. For any vector $x$, $\|x\|$ denotes the Euclidean norm.

\footnote{The source code is archived at \href{https://doi.org/10.5281/zenodo.17083584}{10.5281/zenodo.17083585}.}

\section{Technical Preliminaries}\label{sec_1}
\subsection{Markov Decision Process}
We consider a infinite-horizon MDP, defined as $\mathcal{M}_{\mathcal{S}} = (\mathcal{S}, \mathcal{A},  P_{\mathcal{S}\mathcal{A}}, \gamma, r)$, where $\mathcal{S}$ and $\mathcal{A}$ denote the discrete state and action spaces, respectively. The state-action transition matrix $P_{\mathcal{S}\mathcal{A}}: \mathcal{S} \times \mathcal{A} \rightarrow \Delta(\mathcal{S})$ defines a probability distribution over next states given each state-action pair, while the reward function $r: \mathcal{S} \times \mathcal{A} \rightarrow \mathbb{R}$ specifies the bounded reward received upon transitioning to the next state. The discount factor $\gamma \in (0, 1)$ governs the relative importance of future rewards.

A policy $\pi: \mathcal{S} \rightarrow \Delta(\mathcal{A})$ defines a distribution over actions given each state, and the set of all such policies is denoted by $\Pi_{\mathcal{S}}$, referred to as the policy space. The state transition matrix under policy $\pi$, denoted $P_{\mathcal{S}}^{\pi}$, captures the distribution over next states conditioned only on the current state. Specifically, $P_{\mathcal{S}}^{\pi}(s' \mid s)$ gives the probability of transitioning to state $s'$ from state $s$ under policy $\pi$.

Given a ground MDP $\mathcal{M}_{\mathcal{S}}$ and a policy $\pi: \mathcal{S} \rightarrow \Delta(\mathcal{A})$, the state transition process induces a Markov chain $\mathcal{M}^{\pi}_{\mathcal{S}} = (\mathcal{S}, P_{\mathcal{S}}^{\pi}, \gamma, R^{\pi}_{\mathcal{S}})$, where the expected immediate reward under policy $\pi$ is defined as $R^{\pi}_{\mathcal{S}}(s) = \sum_{a \in \mathcal{A}} \pi(a \mid s) r(s, a)$, and $R^{\pi}_{\mathcal{S}} \in \mathbb{R}^{|\mathcal{S}|}$ is represented as a column vector over states.

The state value function $V^{\pi}_{\mathcal{S}} \in \mathbb{R}^{|\mathcal{S}|}$, which assigns to each state the expected discounted return under policy $\pi$, satisfies the Bellman equation:
$$
V^{\pi}_{\mathcal{S}}(s) = \sum_{a \in \mathcal{A}} \pi(a \mid s) \left[ r(s, a) + \gamma \sum_{\bar{s} \in \mathcal{S}} P_{\mathcal{S}\mathcal{A}}(\bar{s} \mid s, a) V^{\pi}_{\mathcal{S}}(\bar{s}) \right].
$$

In vector form, this can be compactly expressed as:
\begin{equation}\label{value_d_1}
  \begin{split}
    V^{\pi}_{\mathcal{S}} = R^{\pi}_{\mathcal{S}} + \gamma P_{\mathcal{S}}^{\pi} V^{\pi}_{\mathcal{S}} = (I - \gamma P_{\mathcal{S}}^{\pi})^{-1} R^{\pi}_{\mathcal{S}}.
  \end{split}
\end{equation}

Similarly, the state-action value function $Q^{\pi}_{\mathcal{S}}: \mathcal{S} \times \mathcal{A} \rightarrow \mathbb{R}$ satisfies the Bellman equation:
\begin{equation}\label{value_d_2}
  \begin{split}
Q^{\pi}_{\mathcal{S}}(s, a) = r(s, a) + \gamma \sum_{\bar{s} \in \mathcal{S}} P_{\mathcal{S}\mathcal{A}}(\bar{s} \mid s, a) V^{\pi}_{\mathcal{S}}(\bar{s})\text{,}
\end{split}
\end{equation}
which captures the expected return of taking action $a$ in state $s$ and subsequently following policy $\pi$.
The connection between the two functions is further expressed by:
$$
V^{\pi}_{\mathcal{S}}(s) = \sum_{a \in \mathcal{A}} \pi(a \mid s) Q^{\pi}_{\mathcal{S}}(s, a)\text{,}
$$
indicating that the value of a state under policy $\pi$ is the expected value of its state-action values, weighted by the policy's action distribution.

Typically, gien an MDP $\mathcal{M}_{\mathcal{S}}$, a performance function $J_{\mathcal{S}}(\pi) = \mathbb{E}_{s_0 \sim \xi_{\mathcal{S}}}[V^{\pi}_{\mathcal{S}}(s_0)]=\xi_{\mathcal{S}}^{\top} V^{\pi}_{\mathcal{S}} $ is defined to evaluate the quality of a polciy \cite{TRPO}, where $\xi_{\mathcal{S}}$ is initial state distribution.

\subsection{Homomorphic MDPs and Markov Chain}
In the context of MDP, a homomorphism defines a structure-preserving mapping from a ground MDP to a reduced abstract MDP by aggregating state-action pairs that exhibit equivalent behavior in terms of both transition dynamics and reward \cite{ravindran2004approximate}. Formally, given a ground MDP $\mathcal{M}_{\mathcal{S}} = (\mathcal{S}, \mathcal{A}, P_{\mathcal{S}\mathcal{A}}, \gamma, r)$, a homomorphism to an abstract MDP $\mathcal{M}_{\mathcal{S}'} = (\mathcal{S}', \mathcal{A}', P_{\mathcal{S}'\mathcal{A}'}, \gamma, r')$ is defined by a pair of surjective mappings $e: \mathcal{S} \rightarrow \mathcal{S}'$ and $g_s: \mathcal{A} \rightarrow \mathcal{A}'$ for each $s \in \mathcal{S}$, inducing a transformation $h(s,a) = (e(s), g_s(a))$ from $\mathcal{S} \times \mathcal{A}$ to $\mathcal{S}' \times \mathcal{A}'$. The mapping $h$ constitutes a valid MDP homomorphism if for all $(s, a) \in \mathcal{S} \times \mathcal{A}$ and all $s' \in \mathcal{S}$, the following two conditions hold:
\begin{equation}\label{add2_1}
  \begin{split}
    P_{\mathcal{S}'\mathcal{A}'}(e(s'), \mid e(s), g_s(a)) &= \sum_{s'' \in [s'] _e} 
    P_{\mathcal{S}\mathcal{A}}(s'' \mid s, a)
    \\ r'(e(s), g_s(a)) &= r(s, a) \text{,}
  \end{split}
\end{equation}
where $[s']_e = \{s'' \in \mathcal{S} \mid e(s'') = e(s')\}$ denotes the equivalence class of states under $e$. This definition ensures that transition probabilities and immediate rewards in the abstract MDP correctly reflect the aggregate behavior of the corresponding elements in the ground MDP. Importantly, such homomorphic mappings preserve the value structure of the original MDP: optimal policy in the abstract space can be lifted back to policy in the ground space without loss of optimality.

Analogous to homomorphic MDPs, we introduce the definition of a \textbf{homomorphic Markov chain}. While both impose constraints on state transition probabilities and rewards, their respective focuses differ.

\begin{definition}[Homomorphic Markov Chain]\label{de1_2}
  Let $\mathcal{M}_{\mathcal{S}}^{\pi}$ be a ground Markov chain induced by a policy $\pi$ on the ground MDP $\mathcal{M}_{\mathcal{S}}$.
  Let $U$ be an abstract state space with an encoding distribution $\nu(s \mid u)$ that assigns to each abstract state $u \in U$ a probability distribution over ground states $s \in \mathcal{S}$.
  Define the encoding matrix $P_\nu \in \mathbb{R}^{|U| \times |\mathcal{S}|}$ where each row is $\nu(\cdot \mid u)$.
  An abstract Markov chain $\mathcal{M}_{U}^{\mu} = (U, P_{U}^{\mu}, \gamma, R_{U}^{\pi, \nu})$ is defined under an abstract policy $\mu \in \Pi_{U}$, where $\Pi_{U}$ denotes the policy space associated with the abstract state space $U$.
  We say that $\mathcal{M}_{U}^{\mu}$ is a homomorphic Markov chain of the ground Markov chain $\mathcal{M}_{\mathcal{S}}^{\pi}$ if the following condition holds:
  \begin{equation}\label{hmc_1_1}
    \begin{split}
      P_{U}^{\mu} P_{\nu} &= P_{\nu} P_{\mathcal{S}}^{\pi}, \\
      R_{U}^{\pi, \nu} &= P_{\nu} R_{\mathcal{S}}^{\pi}.
    \end{split}
  \end{equation}
\end{definition}

Finally, we present the definition of optimal policy equivalence, a concept that bears a certain relation to Optimal Coupling \cite{ferns2014bisimulation}.
\begin{definition}[Optimal Policy Equivalence]\label{de1_4}
  Given a finite abstract state space $U$.
 Let $\mathcal{E} = \ \{\mathcal{M}_{\mathcal{S}}^{\pi} : \pi \in \Pi_{\mathcal{S}}\}$ denote the set of ground Markov chains induced by all policies $\pi$ in the ground policy space $\Pi_{\mathcal{S}}$, and let $\mathcal{Q} = \{\mathcal{M}_{U}^{f(\pi)} : \pi \in \Pi_{\mathcal{S}}\}$ be the corresponding set of abstract Markov chains, where $f: \Pi_{\mathcal{S}} \to \Pi_U$ is a policy mapping from the ground to the abstract policy space.
  The notion of optimal policy equivalence requires that for any optimal policy $\pi^* \in \Pi_{\mathcal{S}}$ with associated optimal value function in $\mathcal{E}$, the mapped policy $f(\pi^*)$ is also optimal with respect to the value function in $\mathcal{Q}$, and conversely, the optimal policy in $\mathcal{Q}$ correspond to optimal policy in $\mathcal{E}$ under the inverse mapping.  
\end{definition}

\section{Homomorphic Mapping and Markov Chain}\label{sec_2}
This section aims to investigate sufficient conditions under which optimal policy equivalence holds. 
We begin by analyzing the properties of value functions under homomorphic Markov chains, with a particular focus on their relationship to the value functions of the corresponding ground Markov chains.
Next, we extend these properties to optimal policy value functions and introduce the notion of homomorphic mappings as a replacement for homomorphic MDPs.
Finally, leveraging homomorphic mappings, we derive sufficient conditions for optimal policy equivalence.

\subsection{Value Structure and Optimality in Homomorphic Markov Chains}
In this subsection, we aim to investigate the value function properties of homomorphic Markov chains. A key result is that the value function of a homomorphic Markov chain bears a linear relationship to that of the ground Markov chain, which serves as a foundational step toward establishing optimal policy equivalence. 

\begin{lemma}[Matrix Geometric Series \cite{bao2008infinite}(pp. 328)]\label{lemma1}
  If matrix $A$ satisfies that $\lim_{t \rightarrow \infty }A^{t}= {0}$, then $(I - A)^{-1} = \sum_{t = 0}^{\infty} A^{t}$.
\end{lemma}

\begin{theorem}\label{theorem1}
If $\mathcal{M}_{U}^{\mu}$ is a homomorphic Markov chain of the ground Markov chain $\mathcal{M}_{\mathcal{S}}^{\pi}$, then their value functions are related by:
$
V_{U}^{\mu} = P_{\nu} V_{\mathcal{S}}^{\pi}.
$
\end{theorem}

\begin{proof}
  According to Equation \eqref{value_d_1} and the result of Lemma \ref{lemma1}, we replace the term $(I - \gamma P_{U}^{\mu})^{-1}$ by its equivalent infinite series representation:
  \begin{equation}\label{theorem1_1_a}
    \begin{split}
      V^{\mu}_{U} & = (I - \gamma P_{U}^{\mu})^{-1} R^{\pi,\nu}_{U} \\
      &= \lim_{T \rightarrow \infty} \sum_{t=0}^{T} (\gamma P_{U}^{\mu})^{t} R^{\pi,\nu}_{U} \text{.}
    \end{split}
  \end{equation}

  From Equation \eqref{hmc_1_1}, we obtain:
  \begin{equation}\label{theorem1_2_a}
    \begin{split}
      V^{\mu}_{U} & = \lim_{T \rightarrow \infty} \sum_{t=0}^{T} (\gamma P_{U}^{\mu})^{t} P_{\nu} R^{\pi}_{\mathcal{S}} \text{.}
      \\ &= P_{\nu} R^{\pi}_{\mathcal{S}}+ \lim_{T \rightarrow \infty} \sum_{t=1}^{T} \gamma^{t} (P_{U}^{\mu})^{t} P_{\nu} R^{\pi}_{\mathcal{S}} 
      \\ & = P_{\nu} R^{\pi}_{\mathcal{S}} + \lim_{T \rightarrow \infty} \sum_{t=1}^{T} \gamma^{t} (P_{U}^{\mu})^{t - 1} \left(P_{U}^{\mu}P_{\nu}\right) R^{\pi}_{\mathcal{S}} \text{.}
    \end{split}
  \end{equation}

  Since $P_{U}^{\mu} P_{\nu} = P_{\nu} P_{\mathcal{S}}^{\pi}$, the following equation holds:
  \begin{equation}\label{theorem1.2}
    \begin{split}
      V^{\mu}_{U} 
      &= P_{\nu} R^{\pi}_{\mathcal{S}} + \lim_{T \rightarrow \infty} \sum_{t=1}^{T} \gamma^{t} (P_{U}^{\mu})^{t - 1} \left( P_{\nu} P_{\mathcal{S}}^{\pi} \right) R^{\pi}_{\mathcal{S}} \text{.}
  \end{split}
  \end{equation}

  Similarly, for the second term, we can repeatedly apply $P_{U}^{\mu} P_{\nu} = P_{\nu} P_{\mathcal{S}}^{\pi}$ to express $P_{U}^{\mu}$ in terms of $P_{\mathcal{S}}^{\pi}$:
  \begin{equation}\label{theorem1.2_a}
    \begin{split}
      V^{\mu}_{U} &= P_{\nu} R^{\pi}_{\mathcal{S}} + \lim_{T \rightarrow \infty} \sum_{t=1}^{T} \gamma^{t} (P_{U}^{\mu})^{t - 2} \left(P_{U}^{\mu} P_{\nu} \right) P_{\mathcal{S}}^{\pi} R^{\pi}_{\mathcal{S}} \\
      &= P_{\nu} \lim_{T \rightarrow \infty} \sum_{t=0}^{T}  (\gamma P_{\mathcal{S}}^{\pi})^{t}  R^{\pi}_{\mathcal{S}}\text{.}
  \end{split}
  \end{equation}

  Finally, by applying Lemma 1 once again, we obtain:
  \begin{equation}\label{theorem1.2_b}
    \begin{split}
      V^{\mu}_{U}
      &= P_{\nu} (I - \gamma P_{\mathcal{S}}^{\pi})^{-1} R^{\pi}_{\mathcal{S}} \\
      &= P_{\nu} V^{\pi}_{\mathcal{S}}\text{.}
  \end{split}
  \end{equation}
  
\end{proof}

Theorem \ref{theorem1} establishes that the value functions of a homomorphic Markov chain and its corresponding ground Markov chain are linearly related. This highlights a strong connection between the two chains from the perspective of value functions. Next, We introduce the notion of a \textbf{homomorphic mapping} to further investigate the relationship between all policy-induced ground Markov chains and their corresponding homomorphic Markov chains.
For convenience, let ${\Pi}_{\mathcal{S}}$ and ${\Pi}_{U}$ denote the policy spaces over the ground state spaces and abstract state spaces, respectively.

\begin{definition}[Homomorphic Mapping]\label{def_homomorphic_mapping}
  Given an MDP $\mathcal{M}_{\mathcal{S}}$ and an arbitrary encoding distribution $\nu$, a mapping $f_{\nu}: {\Pi}_{\mathcal{S}} \rightarrow {\Pi}_{U}$ is called a \textbf{homomorphic mapping} if, for every $\pi \in {\Pi}_{\mathcal{S}}$, the corresponding abstract Markov chain $\mathcal{M}_{U}^{f_{\nu}(\pi)}$ is a homomorphic Markov chain of $\mathcal{M}_{\mathcal{S}}^{\pi}$.  
\end{definition}

\begin{theorem}\label{theorem2}
  Given an MDP $\mathcal{M}_{\mathcal{S}}$ and an encoding matrix $P_{\nu}$, if there exists a homomorphic mapping $f_{\nu}: \Pi_{\mathcal{S}} \rightarrow \Pi_{U}$, then  $f_{\nu}(\pi^*)$ in $\Pi_{U}$ is optimal, where  $\pi^*$ is the optimal policy in $\mathcal{M}_{\mathcal{S}}$, and vice versa-establishing \textbf{optimal policy equivalence}.
  Moreover, given an initial abstract state distribution $\xi_{U}$ such that $\xi_{\mathcal{S}}^\top = \xi_{U}^\top P_{\nu}$, the performance of the abstract policy matches that of the ground policy for any $\pi \in \Pi_{\mathcal{S}}$, i.e.,
  \begin{equation}\label{theorem2_1}
    \begin{split}
  J_{U}(f_{\nu}(\pi)) = J_{\mathcal{S}}(\pi)\text{.}
  \end{split}
  \end{equation}
\end{theorem}
\begin{proof}
 First, we show that if a policy is optimal in the ground state space ${\Pi}_{\mathcal{S}}$, then its image under the homomorphic mapping is also optimal in the abstract state space $\Pi_{U}$.

  Let \( \pi^{*} \) denote the optimal policy for the ground MDP, such that for all \( \pi \in\Pi_{\mathcal{S}}\) and all \( s \in \mathcal{S} \),
  $
  V_{\mathcal{S}}^{\pi^{*}}(s) \geq V_{\mathcal{S}}^{\pi}(s) \text{.}
  $
  For any vector $\beta$, let $\beta(i)$ represent the $i$-th element of $\beta$. If $\forall i \in \{1, 2,  \ldots, |\mathcal{S}|\}$, $\beta(i) \geq 0$, then:  
  \begin{equation}\label{theorem_policy_1_0_1}
      \begin{split}
        \sum_{i=1}^{|\mathcal{S}|} \beta(i) V_{\mathcal{S}}^{\pi^{*}}(s_i)
        & \geq \sum_{i=1}^{|\mathcal{S}|} \beta(i) V_{\mathcal{S}}^{\pi}(s_i) \text{.}
    \end{split}
  \end{equation}

  Based on the above result, if $\pi^{*}$ is the optimal policy in ground MDP $\mathcal{M}_{\mathcal{S}}$, $\forall u \in U$, we have:
  \begin{equation}\label{theorem_policy_1.0}
    \begin{split}
      V_{U}^{f_{\nu}(\pi^{*})}(u) &= (P_{\nu} V_{\mathcal{S}}^{\pi^{*}})(u)
      \\ & \geq (P_{\nu} V_{\mathcal{S}}^{\pi})(u)
      \\ & = V_{U}^{f_{\nu}(\pi)}(u) \text{.}
  \end{split}
  \end{equation}

  Because $f_{\nu}(\pi) \in {\Pi}_{U}$, it follows that 
  \begin{equation}\label{theorem_policy_1.1}
    \begin{split}
      V_{U}^{f_{\nu}(\pi^{*})}(u) \geq V_{U}^{f_{\nu}(\pi)}(u)\text{.}
  \end{split}
  \end{equation}

  Since the above result holds for all policies in $\Pi_{\mathcal{S}}$ and $f_{\nu}$ is a surjective mapping from set $\Pi_{\mathcal{S}}$ to set ${\Pi}_{U}$, it follows that $V_{U}^{f_{\nu}(\pi^{*})}(u) \geq V_{U}^{f_{\nu}(\pi)}(u)$ holds for $\pi \in {\Pi}_{\mathcal{S}}$.

  Conversely, we use a proof by contradiction to show that if $f(\tilde{\pi})$ is an optimal policy in the  $\Pi_{U}$, then $\tilde{\pi}$ must also be optimal in $\Pi_{\mathcal{S}}$.

Assume, for the sake of contradiction, that $\tilde{\pi}$ is not optimal, i.e., $\exists\, \pi^* \in \Pi_{\mathcal{S}}$ such that $V_{\mathcal{S}}^{\pi^*} \geq V_{\mathcal{S}}^{\tilde{\pi}}$.
Since the value functions are preserved under the homomorphic mapping, i.e., $P_{\nu} V_{\mathcal{S}}^{\pi} = V_{U}^{f_{\nu}(\pi)}$ for all $\pi$, it follows that

$$
V_{U}^{f_{\nu}(\pi^{*})}(u) \geq V_{U}^{f_{\nu}(\tilde{\pi})}(u), \forall u \in U,
$$
which contradicts the assumption that $f(\tilde{\pi})$ is optimal in the abstract space.
Hence, $\tilde{\pi}$ must be an optimal policy in the ground MDP.  Thus, we prove that the existence of a homomorphic mapping necessarily implies optimal policy equivalence.

Finally, we establish the equivalence of the performance functions. According to the definition of the performance function, we have:
\begin{equation}\label{theorem2_proof_1}
  \begin{split}
    J_{\mathcal{S}}(\pi) & = \xi_{\mathcal{S}}^{\top} V^{\pi}_{\mathcal{S}} 
    \text{.}
  \end{split}
\end{equation}

Given the condition $\xi_{\mathcal{S}}^{\top} = \xi_{U}^{\top} P_{\nu}$, we have:
\begin{equation}\label{theorem2_proof_2}
  \begin{split}
    J_{\mathcal{S}}(\pi)  = \xi_{U}^{\top} P_{\nu} V^{\pi}_{\mathcal{S}} 
    \text{.}
  \end{split}
\end{equation}

According to the conclusion of Theorem \ref{theorem1}, $V_{U}^{\mu} = P_{\nu} V_{\mathcal{S}}^{\pi}$, then:
\begin{equation*}\label{theorem2_proof_3}
  \begin{split}
    J_{\mathcal{S}}(\pi) & =  \xi_{U}^{\top} V_{U}^{\mu}
    = J_{U}(f_{\nu}(\pi))
    \text{.}
  \end{split}
\end{equation*}

\end{proof}

Theorem \ref{theorem2} shows that if a homomorphic mapping $f_{\nu}: {\Pi}_{\mathcal{S}} \rightarrow {\Pi}_{U}$ exists, then $\mathcal{E} = \ \{\mathcal{M}_{\mathcal{S}}^{\pi} : \pi \in \Pi_{\mathcal{S}}\}$ and  $\mathcal{Q} = \{\mathcal{M}_{U}^{f(\pi)} : \pi \in \Pi_{\mathcal{S}}\}$ satisfies optimal policy equivalence. Moreover, for the other result concerning the performance function in Theorem \ref{theorem2}, the condition $\xi_{\mathcal{S}}^\top = \xi_{U}^\top P_{\nu}$ is readily satisfied. This is because, for any probability vector $x \in \mathbb{R}^{|\mathcal{S}|}$ and any policy $\pi \in \Pi_{\mathcal{S}}$, the inequality
$$
x^{\top} V_{\mathcal{S}}^{\pi^{*}} \geq x^{\top} V_{\mathcal{S}}^{\pi}
$$
always holds. Therefore, the choice of initial distribution $\xi_{\mathcal{S}}$ does not affect the optimal policy. In other words, for any encoding matrix $P_{\nu}$, as long as $\xi_{\mathcal{S}}^\top$ lies within the row space of $P_{\nu}$, there necessarily exists a $\xi_{U}^\top$ such that $\xi_{\mathcal{S}}^\top = \xi_{U}^\top P_{\nu}$.


\subsection{Characterizing the Existence of Homomorphic Mappings}

In the previous subsection, we established that the existence of a homomorphic mapping serves as a sufficient condition for optimal policy equivalence. Building on this result, the goal of the current subsection is to investigate the necessary and sufficient conditions for the existence of such a homomorphic mapping.

\begin{definition}\label{de1_3}
  For the ground MDP $\mathcal{M}_{\mathcal{S}}$ with state space \(\mathcal{S}\) and action space \(\mathcal{A}\), we define the elementary transition vectors as:
  \[
  \alpha_{i,j} := P_{\mathcal{S}\mathcal{A}}(\cdot|s_i,a_j) \in \mathbb{R}^{|\mathcal{S}|}, \quad  s_i \in \mathcal{S}, a_j \in \mathcal{A} \text{.}
  \]

  Let \(\mathcal{F} = \{\alpha^{(1)},...,\alpha^{(r)}\}\) denote a maximal linearly independent subset (basis) of \(\{\alpha_{i,j}:  \forall s_i \in \mathcal{S}, a_j \in \mathcal{A}\}\), where \(r = \text{rank}(\{\alpha_{i,j}\}) \leq |\mathcal{S}|\).
\end{definition}

\begin{theorem}\label{theorem4}

Given a ground MDP $\mathcal{M}_{\mathcal{S}}$ and encoding matrix $P_{\nu}$, the homomorphic mapping $f_{\nu}$ exists  \textbf{if and only if} the row space of $P_\nu$ contains $\text{span}(\mathcal{F})$. Under this condition, for any policy $\pi \in \Pi_{\mathcal{S}}$, $(U, P_\nu C^{\pi}, R^{\pi, \nu}_{\mathcal{S}}, \gamma)$ is the homomorphic Markov chain of ground Markov chain $P_{\mathcal{S}}^{\pi}$, where $C^{\pi} = P_{\mathcal{S}}^{\pi} P_{\nu}^{\dagger}$ and $P_{\nu}^{\dagger} = P_{\nu}^{\top} (P_{\nu}P_{\nu}^{\top})^{-1}$ is the Moore-Penrose pseudoinverse of $P_{\nu}$.

\end{theorem}

\begin{proof}We begin by proving the necessary and sufficient condition.
  
  Necessity:
  Assume that $P_{U}^{\mu}$ exists such that $P_{U}^{\mu} P_{\nu} = P_{\nu} P_{\mathcal{S}}^{\pi}$ holds for all $\pi$. Note that each $P_{\mathcal{S}}^{\pi}$ is a convex combination of the transition vectors $\{\alpha_{i,j}\}$. Hence, the set of all such products $P_{\nu} P_{\mathcal{S}}^{\pi}$ lies within the projection of the linear combinations of $\{\alpha_{i,j}\}$ under $P_\nu$. In order for $P_{U}^{\mu} P_{\nu}$ to match $P_{\nu} P_{\mathcal{S}}^{\pi}$, the row space of $P_\nu$ must span all possible linear combinations of $\{\alpha_{i,j}\}$, or at least a basis of them - i.e., $\mathcal{F}$.
  Thus, $\operatorname{Row}(P_\nu)$ must contain $\operatorname{span}(\mathcal{F})$.
  
  Sufficiency:
  Assume $\operatorname{Row}(P_\nu) \supseteq \operatorname{span}(\mathcal{F})$. Then, for any policy $\pi$, its induced transition matrix $P_{\mathcal{S}}^{\pi}$ can be written as a linear combination of $\mathcal{F}$, and hence any column $P_{\mathcal{S}}^{\pi} v$, for $v \in \mathbb{R}^{|\mathcal{S}|}$, lies within $\operatorname{span}(\mathcal{F})$.
  Since $P_\nu$ acts on this space (and includes it in its row space), for any such $\pi$, there exists a linear operator $P_U^{\mu}$ defined on the abstract space such that:

  \begin{equation}\label{theorem_MPP_add_6}
    \begin{split}
      P_U^{\mu} P_{\nu} = P_{\nu} P_{\mathcal{S}}^{\pi}
      \text{.}
  \end{split}
  \end{equation}
  That is, the dynamics under $P_{\mathcal{S}}^{\pi}$ can be lifted through $P_\nu$ via a corresponding abstract dynamics $P_U^{\mu}$.

As a result, there exists a matrix $C^{\pi} \in \mathbb{R}^{|\mathcal{S}| \times |U|}$ such that $C^{\pi} P_\nu = P^\pi_\mathcal{S}$. Therefore, we next derive a closed-form solution for $C^{\pi}=P_{\mathcal{S}}^{\pi} P_{\nu}^{\dagger}$.
To verify this result, we substitute $C^{\pi} = P_{\mathcal{S}}^{\pi} P_{\nu}^{\dagger}$ into Equation \eqref{theorem_MPP_add_4}, yielding:
\begin{equation}\label{theorem_MPP_add_3}
  \begin{split}
    P_{U}^{f_{\nu}(\pi)} P_\nu &= P_\nu (C^{\pi} P_\nu)
    \\ & = 
    P_\nu P_{\mathcal{S}}^{\pi} P_{\nu}^{\dagger} P_\nu
    \\ & =
    P_\nu P_{\mathcal{S}}^{\pi} P_{\nu}^{\top} (P_{\nu}P_{\nu}^{\top})^{-1} P_\nu
    \text{.}
\end{split}
\end{equation}
Since the row space of $P_\nu$ contains $\text{span}(\mathcal{F})$, there must exist a matrix $D$ such that $P_{\mathcal{S}}^{\pi} = D P_{\nu}$. Substituting this into the above equation yields:
\begin{equation}\label{theorem_MPP_add_4}
  \begin{split}
    P_{U}^{f_{\nu}(\pi)} P_\nu &= P_\nu D P_{\nu} P_{\nu}^{\top} (P_{\nu}P_{\nu}^{\top})^{-1} P_\nu
    \\ & =
    P_\nu D P_{\nu}
    \\ & = P_\nu P_{\mathcal{S}}^{\pi}
    \text{.}
\end{split}
\end{equation}

According to the conclusion of Theorem \ref{theorem1}, since $R_{U}^{\pi, \nu} = P_{\nu} R_{\mathcal{S}}^{\pi}$ and $P_{U}^{f_{\nu}(\pi)} P_\nu = P_\nu P_{\mathcal{S}}^{\pi}$, it follows that the homomorphic Markov chain of ${\mathcal{M}}_{\mathcal{S}}^{\pi}$ is $(U, P_\nu C^{\pi}, R^{\pi, \nu}_{\mathcal{S}}, \gamma)$.
Combining the fact that every encoding Markov chain admits a corresponding homomorphic Markov chain with the definition of a homomorphic mapping, we conclude that
the row space of $P_\nu$ containing $\text{span}(\mathcal{F})$
is the sufficient and necessary condition for the existence of a homomorphic mapping in the ground MDP.

\end{proof}

\textbf{Finally, we clarify why the sufficient condition in Theorem \ref{theorem4} is more concise and general than the condition (Equation \eqref{add2_1})  presented by Shoshtari et al. \cite{JMLR:v25:23-1415}.} From a definitional standpoint, a homomorphic Markov chain requires only that the transition probabilities be linearly dependent, whereas a homomorphic MDP, as defined in Equation \eqref{add2_1}, requires these probabilities to be exactly equal. This indicates that the structural constraint imposed by Equation \eqref{add2_1} is strictly stronger.

Moreover, Theorem \ref{theorem4} states that the number of abstract states $|U|$ need only be no less than the number of basis functions in $\overline{\mathcal{N}^\nu}$, without requiring one-to-one correspondence with distinct transition distributions.
As a concrete example, suppose that for all $s \in \mathcal{S}$ and $a \in \mathcal{A}$, the transition probabilities $P_{\mathcal{S}\mathcal{A}}(\cdot \mid s, a)$ can be written as convex combinations of two distinct distributions $k_1(\cdot), k_2(\cdot) \in \Delta(\mathcal{S})$. Without loss of generality, assume there exists some $(s_0, a_0)$ such that
\begin{equation}
\begin{split}
P(\cdot \mid s_0, a_0) &= k_3(\cdot) \\
&= 0.5\cdot k_1(\cdot) + 0.5\cdot k_2(\cdot),
\end{split}
\end{equation}
where $k_3 \in \Delta(\mathcal{S})$ and $k_3 \neq k_1, k_2$. According to Theorem \ref{theorem4}, it suffices to define the abstract transition function $v(\cdot \mid u)$ using only the two basis distributions $k_1$ and $k_2$, so that optimal policy equivalence holds even when $|U| = 2$.

In contrast, under the homomorphic MDP condition in Equation \eqref{add2_1}, the mapping $g(s)$ must assign different abstract states to each of $k_1$, $k_2$, and $k_3$, since these represent distinct transition behaviors. This implies that $|U| \geq 3$ is required in that case. Therefore, this example highlights that the sufficient condition in Theorem \ref{theorem4} imposes strictly fewer structural constraints than the homomorphic MDP definition of Shoshtari et al. \cite{JMLR:v25:23-1415}, and is thus both more general and more compact.

\section{Policy Optimization and Performance Analysis under Homomorphic Mapping}\label{sec_3}
In the previous section, we presented a sufficient condition for optimal policy equivalence, which is more compact than the result established by Shoshtari et al. \cite{JMLR:v25:23-1415}. In this subsection, we further explore how to leverage optimal policy equivalence to optimize policies, as well as how to improve policy performance when the sufficient condition is not satisfied.

\begin{algorithm}[tb]
  \caption{: Homomorphic Policy Gradient Algorithm (HPG)}
  \label{alg_HM}
  \begin{algorithmic}[2]
      \STATE Initial the policy ${{\theta}^{0}}$ and $P_{\nu}$
      \STATE According $P_{\mathcal{S}\mathcal{A}}$ calculating $P_{\nu}^{\dagger}$
      \REPEAT
      \STATE $ C^{\pi_{{\theta}^{t}}} = P_{\mathcal{S}}^{\pi_{{\theta}^{t}}} P_{\nu}^{\dagger}$
      \STATE $V_{U}^{f_{\nu}(\pi_{{\theta}^{t}})} =  (I-\gamma P_{\nu} C^{\pi_{{\theta}^{t}}})^{-1} P_\nu R_{\mathcal{S}}^{\pi_{{\theta}^{t}}}$
      \STATE ${\theta}^{t+1} = \arg\max_{\theta^{t}} V_{U}^{f_{\nu}(\pi_{\theta^{t}})}$
      \STATE $t = t + 1$.
      \UNTIL{$\pi_{{\theta}^{t+1}}$ is optimal.}
      \STATE Return $\pi_{{\theta}}^{*}$
  \end{algorithmic}
\end{algorithm}

\subsection{Optimizing Policies in the Abstract Space with Homomorphic Mappings}

According to the results of Theorems \ref{theorem2} and \ref{theorem4}, if the row space of $P_\nu$ contains $\text{span}(\mathcal{F})$, then optimal policy equivalence holds. In other words, 

\begin{equation}\label{HMASA_1_1}
  \begin{split}
    \pi^{*} = \arg\max_{\pi \in \Pi_{\mathcal{S}}} V_{\mathcal{S}}^{\pi} \equiv \arg\max_{\pi \in \Pi_{\mathcal{S}}} V_{U}^{f_{\nu}(\pi)}\text{.}
\end{split}
\end{equation}

Based on the above conclusion, we propose Algorithm \ref{alg_HM}. Algorithm \ref{alg_HM} demonstrates how to optimize a ground MDP policy via the encoding matrix. First, compute the pseudoinverse of $P_{\nu}^{\dagger}$ using matrix $P_{\mathcal{S}\mathcal{A}}$. Then, following the policy iteration procedure, iteratively evaluate the value function and improve the policy. Due to the validity of optimal policy equivalence, this process ultimately converges to the optimal policy of the ground MDP \cite{bertsekas2012dynamic}.

We next analyze the computational complexity of Algorithm \ref{alg_HM}. According to standard matrix computation methods, the computational complexity of calculating the Moore-Penrose pseudoinverse of an $m \times n$ matrix is $\mathcal{O}(m n^2)$ (The most commonly used is the Singular Value Decomposition (SVD) method). For the inversion of an $n \times n$ matrix, the computational complexity of Gaussian elimination is $\mathcal{O}(n^3)$. For Algorithm \ref{alg_HM}, the computational complexity of each policy evaluation (step 4-5) is $\mathcal{O}(|\mathcal{S}||\mathcal{A}| + |U||\mathcal{S}|^{2} + |U|^3)$. In contrast, for standard policy evaluation in the ground MDP, the per-iteration complexity is $\mathcal{O}(|\mathcal{S}||\mathcal{A}| + |\mathcal{S}|^3)$. Clearly, Algorithm \ref{alg_HM} is computationally more efficient when $|U| \ll|\mathcal{S}|$.


We next investigate how to optimize the policy using Equation \eqref{HMASA_1_1}. A straightforward approach is to apply policy gradient methods. Accordingly, we derive the policy gradient in the abstract space for optimizing the ground MDP policy.

\begin{theorem}[Homomorphic Policy Gradient Theorem]\label{new_theorem4}
  The gradient of the corresponding value function $V^{f_{\nu}(\pi_\theta)}_{U}(u)$ with respect to the parameter $\theta$ is given by:
  \begin{equation}\label{HMASA_1_2}
    \begin{split}
      \nabla_{\theta} &V^{f_{\nu}(\pi_\theta)}_{U}(u) 
      \\ & = \mathbb{E}_{X_{t} \sim \eta(\cdot|f(\pi_\theta), u), S_{t} \sim \nu(\cdot|X_{t}), A_{t} \in \pi_\theta(\cdot|S_{t})} \Big[  \nabla_{\theta} \ln \pi_\theta(A_{t}|S_{t})
      \\ &  \cdot
      \big [r(S_{t},A_{t}) + \gamma \mathbb{E}_{Y \sim P_{1}(\cdot | S_{t}, A_{t})} [V^{f_{\nu}(\pi_\theta)}_{U}(Y)] \big] \Big] \text{,}
  \end{split}
  \end{equation}
  where $ \eta(x|u, f(\pi_\theta)) = \sum_{t=0}^{\infty} \gamma^{t} P(X_t = x| X_0 = u, \pi_\theta)$ and $P_{1}(u' | s, a) = \sum_{s' \in \mathcal{S}} P_{\mathcal{S}\mathcal{A}}(s'|s,a) P_{\nu}^{\dagger}(u'|s') $. 
\end{theorem}

\begin{proof}

According to the definition of the value function, we have:
\begin{equation}\label{APG_1_1_a}
  \begin{split}
    \nabla_{\theta} &V^{f_{\nu}(\pi_\theta)}_{U}(u) 
    \\ & = \nabla_{\theta} [R_{\mathcal{S}}^{\pi_\theta, \nu}(u) + \gamma \sum_{u'  \in U} P_{U}^{f_{\nu}(\pi_\theta)}(u'|u) V^{f_{\nu}(\pi_\theta)}_{U}(u')]
    \\ & = \nabla_{\theta}  \big[\sum_{u \in U, s \in \mathcal{S}} \nu(s|u)\pi_\theta(a|s)r(s,a) 
    \\ &  + \gamma \sum_{u'  \in U} P_{U}^{f_{\nu}(\pi_\theta)}(u'|u) V^{f_{\nu}(\pi_\theta)}_{U}(u') \big] \text{.}
\end{split}
\end{equation}

Substituting Equation \eqref{value_d_2} into the above expression, we obtain:
\begin{equation}\label{APG_1_1_b}
  \begin{split}
    \nabla_{\theta} &V^{f_{\nu}(\pi_\theta)}_{U}(u) 
    \\ & = \nabla_{\theta} \sum_{u \in U, s \in \mathcal{S}} \nu(s|u)\pi_\theta(a|s)r(s,a)  
    \\ & + \gamma \sum_{u' \in U} \nabla_{\theta} (P_\nu P_{\mathcal{S}}^{\pi_\theta} P_{\nu}^{\dagger})(u'|u) V^{f_{\nu}(\pi_\theta)}_{U}(u')
    \\ & + \gamma \sum_{u' \in U} (P_\nu P_{\mathcal{S}}^{\pi_\theta} P_{\nu}^{\dagger})(u'|u) \nabla_{\theta}  V^{f_{\nu}(\pi_\theta)}_{U}(u')
    \\ & =  \sum_{s \in \mathcal{S}, a \in \mathcal{A}} [\nu(s|u) \pi_\theta(a|s) r(s,a) \nabla_{\theta} \ln \pi_\theta(a|s)]
    \\ & + \gamma  \sum_{s, s' \in \mathcal{S}, a \in \mathcal{A}, u' \in U} \big [\nu(s|u) \pi_\theta(a|s) P_{\mathcal{S}\mathcal{A}}(s'|s,a) P_{\nu}^{\dagger}(u'|s') 
    \\ & \cdot \nabla_{\theta} \ln \pi_\theta(a | s) V^{f_{\nu}(\pi_\theta)}_{U}(u') \big]
    \\ & + \gamma  \sum_{u' \in U} P_{U}^{f_{\nu}(\pi_\theta)}(u'|u) \nabla_{\theta}  V^{f_{\nu}(\pi_\theta)}_{U}(u').
\end{split}
\end{equation}

Let $P_{1}(u' | s, a) = \sum_{s' \in \mathcal{S}} P_{\mathcal{S}\mathcal{A}}(s'|s,a) P_{\nu}^{\dagger}(u'|s') $. Following this pattern, we obtain:
\begin{equation}\label{APG_1_2}
  \begin{split}
    & =  \sum_{t=0}^{T} \gamma^{t} \sum_{x \in U} \Big[  p(u \rightarrow x, k, \pi_\theta) 
    \\ & \cdot \mathbb{E}_{S_{t} \sim \nu(\cdot|x), A_{t} \in \pi_\theta(|S_{t})} [r(S_{t},A_{t}) \nabla_{\theta} \ln \pi_\theta(A_{t}|S_{t})]
    \\ & + \gamma  \mathbb{E}_{S_{t} \sim \nu(\cdot|x), A_{t} \in \pi_\theta(\cdot|S_{t}), Y \sim P_{1}(\cdot | S_{t}, A_{t})} \big [  \nabla_{\theta} \ln \pi_\theta(A_{t}|S_{t})
    \\ & \cdot V^{f_{\nu}(\pi_\theta)}_{U}(Y) \big] \Big] 
    \\ & = \sum_{x \in U} \eta(x|f(\pi_\theta), u) \mathbb{E}_{S_{t} \sim \nu(\cdot|x), A_{t} \in \pi_\theta(\cdot|S_{t}), Y \sim P_{1}(\cdot | S_{t}, A_{t})} \Big[  
    \\ & \nabla_{\theta} \ln \pi_\theta(A_{t}|S_{t})
    \big [r(S_{t},A_{t}) + \gamma  V^{f_{\nu}(\pi_\theta)}_{U}(Y) \big] \Big]
    \\ &= \mathbb{E}_{X_{t} \sim \eta(\cdot|f(\pi_\theta), u), S_{t} \sim \nu(\cdot|X_{t}), A_{t} \in \pi_\theta(\cdot|S_{t})} \Big[  \nabla_{\theta} \ln \pi_\theta(A_{t}|S_{t})
    \\ &  \cdot
    \big [r(S_{t},A_{t}) + \gamma \mathbb{E}_{Y \sim P_{1}(\cdot | S_{t}, A_{t})} [V^{f_{\nu}(\pi_\theta)}_{U}(Y)] \big] \Big] \text{.}
\end{split}
\end{equation}

\end{proof}

Overall, this subsection explores how to optimize policies through homomorphism mapping. We propose the Homomorphic Policy Gradient (HPG) algorithm, which leverages this structure to improve computational efficiency. We also derive the policy gradient in the abstract space, enabling gradient-based optimization of the ground policy via its abstract representation.

\subsection{Error Analysis and Performance Guarantees under Condition Violation}

\begin{figure}[!h]
  \centering \includegraphics[width=\columnwidth]{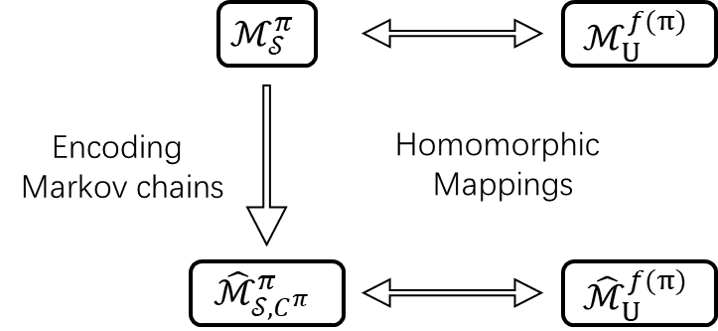}
  \caption{This figure illustrates the relationship between the ground Markov chain, encoding Markov chains, and the homomorphic Markov chain.
  In general, encoding Markov chains may exhibit discrepancies relative to the ground Markov chain. However, there always exists a homomorphic Markov chain corresponding to any encoding Markov chain.   Therefore, encoding Markov chains can serve as a critical bridge connecting the ground MDP and homomorphic mappings.
  }
  \label{Figure1}
\end{figure}

In the previous subsection, we examined homomorphisms and state aggregation under idealized assumptions. In this subsection, we investigate how to utilize homomorphic mappings to optimize policies when the row space of $P_\nu$ does not contain $\text{span}(\mathcal{F})$. Clearly, in the absence of this condition, optimal policy equivalence no longer holds, implying the introduction of value function approximation errors. Accordingly, we first derive an upper bound on the performance gap and then provide a lower bound on the performance of the policy in the ground MDP.

\begin{definition}[Encoding Markov Chain]
  Given a ground MDP, an encoding matrix $P_{\nu}$, a matrix $C^{\pi}\in\mathbb{R}^{|\mathcal{S}|\times|U|}$, and a policy $\pi \in \Pi_\mathcal{S}$, we define $\hat{\mathcal{M}}_{\mathcal{S}, \nu}^{\pi}=(\mathcal{S}, C^{\pi}P_{\nu},R_{\mathcal{S}}^{\pi},\gamma)$ as an encoding Markov chain of ground Markov chain ${\mathcal{M}}_{\mathcal{S}}^{\pi}$, where $C^{\pi}=P_{\mathcal{S}}^{\pi} P_{\nu}^{\dagger}$. 
\end{definition}

As illustrated in Figure \ref{Figure1}, for each ground Markov chain, we can associate an encoding Markov chain that approximates it with some error. We view the encoding Markov chain as a bridge connecting the ground Markov chain to a potential homomorphic Markov chain.
For any ground Markov chain ${\mathcal{M}}_{\mathcal{S}}^{\pi} = (\mathcal{S}, P_{\mathcal{S}}^{\pi}, R^{\pi}_{\mathcal{S}}, \gamma)$ and its corresponding encoding Markov chain $\hat{\mathcal{M}}_{\mathcal{S}, \nu}^{\pi} = (\mathcal{S}, C^{\pi} P_\nu, R^{\pi}_{\mathcal{S}}, \gamma)$. The homomorphic Markov chain induced by the encoding matrix is denoted as ${\mathcal{M}}_{U}^{f_{\nu}(\pi)} = (U, P_\nu C^{\pi}, R^{\pi, \nu}_{\mathcal{S}}, \gamma)$.  Following the previous notation, we denote the value functions of $\hat{\mathcal{M}}_{\mathcal{S}, \nu}^{\pi}$  as $\hat{V}_{\mathcal{S},\nu}^{\pi}$ and $V_{U}^{f_{\nu}(\pi)} = P_{\nu} \hat{V}_{\mathcal{S},\nu}^{\pi}$.

\begin{theorem}[Policy Optimization Lower Bound Theorem]\label{POLB} $\\$
Assume there exists a initial state distribution $\xi_{U}$ such that $\xi_{U}^{\top} P_{\nu} = \xi_{\mathcal{S}}^{\top}$.
The lower bound on policy performance in the ground MDP satisfies:
\begin{equation}\label{polbt_1}
  \begin{split}
    J_{\mathcal{S}}(\tilde{\pi}) \geq J_{U}(f_{\nu}(\tilde{\pi})) - \frac{\mathcal{k}g(\tilde{\pi}, \nu ) \mathcal{k}}{ 1 - \gamma} \text{,}
  \end{split}
\end{equation}
where  $\frac{\mathcal{k} g(\pi, \nu ) \mathcal{k}}{1 - \gamma} = \frac{\mathcal{k} P_{\nu} P_{\mathcal{S}}^{\pi} \hat{V}_{\mathcal{S},\nu}^{\pi} - P_{U}^{f_{\nu}(\pi)} {V}_{U}^{f_{\nu}(\pi)} \mathcal{k} }{ 1 - \gamma } $ is the upper bound on the performance discrepancy between policy $\pi$ in the ground MDP and its image $f(\pi)$ in the abstract space.

\end{theorem}

\begin{proof}
Our proof proceeds as follows:

We first show that $  \mathcal{k} J_{\mathcal{S}}(\pi) - J_{U}(f_{\nu}(\pi)) \mathcal{k} \leq \frac{\mathcal{k} g(\pi, \nu ) \mathcal{k}}{1 - \gamma}$. 
Based on the ground Markov chain and its corresponding encoding Markov chain, we have:
  \begin{equation}\label{HMASA_2_1}
    \begin{split}
      \mathcal{k} J_{\mathcal{S}} ({\pi})  & - J_{U} ({f_{\nu}(\pi)} ) \mathcal{k}  
      \\ & = \mathcal{k} \xi_{\mathcal{S}}^{\top} V_{\mathcal{S}}^{\pi} - \xi_{U}^{\top} \hat{V}_{U}^{f_{\nu}{\pi}}  \mathcal{k}
      \\ & = \mathcal{k} \xi_{\mathcal{S}}^{\top} V_{\mathcal{S}}^{\pi} - \xi_{U}^{\top} P_\nu \hat{V}_{\mathcal{S},\nu}^{\pi}  \mathcal{k}
      \\ & = \mathcal{k}  \xi_{\mathcal{S}}^{\top} [(I - \gamma P_{\mathcal{S}}^{\pi} )^{-1} - (I - \gamma C^{\pi} P_\nu)^{-1} ]R_{\mathcal{S}}^{\pi} \mathcal{k} \text{.}
  \end{split}
  \end{equation}
  
  If the matrices $(I - A)$ and $(I - B)$ are invertible, then it follows that 
  \begin{equation}\label{lemma3_1}
    \begin{split}
      (I - A)^{-1} & (A - B) (I - B)^{-1} 
      \\ &= (I - A)^{-1} [(A - I) + (I - B)] (I - B)^{-1}
      \\ & = (I - A)^{-1} - (I - B)^{-1} \text{.}
    \end{split}
  \end{equation}

  Substituting Equation \eqref{lemma3_1} into Equation \eqref{HMASA_2_1}, we obtain:
  \begin{equation}\label{HMASA_2_2}
    \begin{split}
      &\mathcal{k}  J_{\mathcal{S}} ({\pi})   - J_{U} ({f_{\nu}(\pi)} ) \mathcal{k}  
      \\ & = \mathcal{k} \xi_{\mathcal{S}}^{\top} (I - \gamma P_{\mathcal{S}}^{\pi} )^{-1}(P_{\mathcal{S}}^{\pi} - C^{\pi} P_\nu)(I - \gamma C^{\pi} P_\nu)^{-1} R_{\mathcal{S}}^{\pi} \mathcal{k} 
      \\ & = \mathcal{k} \xi_{U}^{\top} P_{\nu}  \sum_{k=0}^{\infty}(\gamma P_{\mathcal{S}}^{\pi})^{k} (P_{\mathcal{S}}^{\pi} - C^{\pi} P_\nu)(I - \gamma C^{\pi} P_\nu)^{-1} R_{\mathcal{S}}^{\pi} \mathcal{k} 
      \\ & = \mathcal{k} \xi_{U}^{\top}  \sum_{k=0}^{\infty}(\gamma P_{U}^{f(\pi)})^{k} P_{\nu} (P_{\mathcal{S}}^{\pi} - C^{\pi} P_\nu)(I - \gamma C^{\pi} P_\nu)^{-1} R_{\mathcal{S}}^{\pi} \mathcal{k} 
      \\ & \leq \mathcal{k} \xi_{U}^{\top}  \sum_{k=0}^{\infty}(\gamma P_{U}^{f(\pi)})^{k} \mathcal{k} \mathcal{k} P_{\nu} (P_{\mathcal{S}}^{\pi} - C^{\pi} P_\nu)   \hat{V}_{\mathcal{S},\nu}^{\pi} \mathcal{k}
      \text{,}
  \end{split}
  \end{equation}
  where the inequality follows from the law of cosines.
  For the right-hand side of the above equation, we have:
  \begin{equation}\label{HMASA_2_3}
    \begin{split}
      \mathcal{k} \xi_{U}^{\top}  \sum_{k=0}^{\infty}(\gamma P_{U}^{f(\pi)})^{k}  \mathcal{k} 
       & \leq   \lim_{T \rightarrow \infty} \sum_{t=0}^{T} \gamma^{t}  \mathcal{k} \xi_{U}^{\top} (P_{U}^{f(\pi)})^{t}  \mathcal{k} 
      \text{.}
  \end{split}
  \end{equation}
  For any probability vector $x$, it holds that $\mathcal{k} x \mathcal{k} \leq 1$.
  Therefore, we have:
  \begin{equation}\label{HMASA_2_4}
    \begin{split}
      \mathcal{k}  J_{\mathcal{S}} ({\pi}) &  - J_{U} ({f_{\nu}(\pi)} ) \mathcal{k} 
      \\ & \leq  \left( \lim_{T \rightarrow \infty} \sum_{t=0}^{T} \gamma^{t} \right) \mathcal{k} P_{\nu} (P_{\mathcal{S}}^{\pi} - C^{\pi} P_\nu)   \hat{V}_{\mathcal{S},\nu}^{\pi} \mathcal{k}
      \\ & =
      \frac{1}{1 - \gamma} \mathcal{k} P_{\nu} (P_{\mathcal{S}}^{\pi} - C^{\pi} P_\nu)   \hat{V}_{\mathcal{S},\nu}^{\pi} \mathcal{k}
      \text{.}
  \end{split}
  \end{equation}

  By the definition of the value function, we have:
  \begin{equation}\label{lemma5_1}
    \begin{split}
      P_{\nu} C^{\pi} P_\nu \hat{V}_{\mathcal{S},\nu}^{\pi} &= P_{\nu} P_{\mathcal{S}}^{\pi} P_{\nu}^{\dagger} P_\nu \hat{V}_{\mathcal{S},\nu}^{\pi} 
      \\ &= P_{U}^{f_{\nu}(\pi)} P_{\nu}  P_{\nu}^{\dagger} P_\nu \hat{V}_{\mathcal{S},\nu}^{\pi}
      \\ & = P_{U}^{f_{\nu}(\pi)} {V}_{U}^{f_{\nu}(\pi)} \text{.}
    \end{split}
  \end{equation}

  Substituting Equation \eqref{lemma5_1} into Equation \eqref{HMASA_2_4}, we obtain:
  \begin{equation}\label{add_HMASA_2_5}
    \begin{split}
      \mathcal{k}  J_{\mathcal{S}} ({\pi})  - J_{U} ({f_{\nu}(\pi)} ) \mathcal{k} 
      & \leq 
      \frac{\mathcal{k} P_{\nu} P_{\mathcal{S}}^{\pi} \hat{V}_{\mathcal{S},\nu}^{\pi} - P_{U}^{f_{\nu}(\pi)} {V}_{U}^{f_{\nu}(\pi)} \mathcal{k} }{ 1 - \gamma }
      \\ & =  \frac{\mathcal{k}g(\tilde{\pi}, \nu ) \mathcal{k}}{ 1 - \gamma }
      \text{.}
  \end{split}
  \end{equation}

According to Equation \eqref{add_HMASA_2_5} and the triangle inequality, it holds that
\begin{equation}\label{HMASA_2_6}
  \begin{split}
    J_{\mathcal{S}}(\tilde{\pi}) & = J_{\mathcal{S}}(\tilde{\pi})  + J_{U}(f_{\nu}(\tilde{\pi})) - J_{U}(f_{\nu}(\tilde{\pi}))
    \\ & \geq J_{U}(f_{\nu}(\tilde{\pi})) - \mathcal{k} J_{\mathcal{S}}(\tilde{\pi}) - J_{U}(f_{\nu}(\tilde{\pi})) \mathcal{k}
    \\ & \geq J_{U}(f_{\nu}(\tilde{\pi})) - \frac{\mathcal{k}g(\tilde{\pi}, \nu ) \mathcal{k}}{ 1 - \gamma }
    \text{.}
\end{split}
\end{equation} 

\end{proof}

Theorem \ref{POLB} states that, given an encoding matrix, the value function error between a ground Markov chain and its corresponding encoding Markov chain is bounded above by $\frac{\mathcal{k}g(\tilde{\pi}, \nu ) \mathcal{k}}{ 1 - \gamma }$. 
Moreover, $J_{U}(f_{\nu}(\tilde{\pi})) - \frac{\mathcal{k}g(\tilde{\pi}, \nu ) \mathcal{k}}{ 1 - \gamma }$ represents a lower bound on the policy performance. In other words, when the row space of $P_\nu$ does not contain $\text{span}(\mathcal{F})$, we regard the policy performance lower bound as the objective function to be optimized. 
Finally, we derive a feasible gradient ascent direction for the optimization variables. The optimization involves not only the policy $\pi$, but also the encoding matrix $P_{\nu}$. Let $\theta$, and $\omega$ denote the parameters of the policy $\pi$, and the encoding matrix $P_{\nu}$, respectively.


\begin{lemma}\label{lemm2}
  Let $A$ be an $n \times n$ invertible square matrix, $\mathbf{W}$ be the inverse of $A$, and $F({A})$ is an $n \times n$-variate and differentiable function with respect to ${A}$, then the partial differentials of $F$ with respect to ${A}$ and $\mathbf{W}$ satisfy
  $$ \frac{\partial F}{\partial {A}} = -{A}^{-\top} \frac{\partial F}{\partial \mathbf{W}} {A}^{-\top} \text{,}$$
  where ${A}^{-\top} = (A^{\top})^{-1}$.
  The conclusion follows from reference \cite{petersen2012matrix} (section 2.3, pp. 10). 
\end{lemma}

\begin{theorem}[Encoding Matrix Gradient Theorem]\label{theorem5}
  The gradient with respect to the parameter $\theta$ and $\omega$ are given by:
  \begin{equation}\label{HMASA_4_1}
    \begin{split}
       \nabla_{\theta} & \Big[J_{U}(f_{\nu}(\pi_{\theta})) -  \frac{\mathcal{k} g({\pi_\theta}, \nu) \mathcal{k}}{1 - \gamma}\Big]  
      \\ & = \mathbb{E}_{U_{0} \sim \xi_{U}} \Big[\nabla_{\theta} V_{U}(f_{\nu}(\pi_{\theta}))\Big] 
      \\ & - \sum_{u \in U, s, s' \in\mathcal{S}}\frac{2 g({\pi_\theta}, \nu)(u)}{\mathcal{k}  g({\pi_\theta}, \nu) \mathcal{k}} \Big[ \nu(s|u) 
      \big[ P_{\mathcal{S}}^{\pi_{\theta}}(s'|s) \nabla_{\theta} \hat{V}_{\mathcal{S},\nu}^{\pi_{\theta}}(s') 
      \\ & +  \sum_{a \in \mathcal{A}} P_{\mathcal{S}\mathcal{A}}(s'|s, a) \hat{V}_{\mathcal{S},\nu}^{\pi_{\theta}}(s') \nabla_{\theta} \pi_{\theta}(a|s) \big] 
      \\ & - \sum_{u \in U}  \nu(s|u)  \frac{1}{\gamma} \nabla_{\theta} (\hat{V}_{\mathcal{S},\nu}^{\pi_{\theta}} - R_{\mathcal{S}}^{\pi_{\theta}})(s) \Big]
  \end{split}
  \end{equation}
  and
  \begin{equation}\label{HMASA_4_2}
    \begin{split}
      &\nabla_{\omega} \Big[J_{U}(f_{\nu_{\omega}}(\pi)) -  \frac{\mathcal{k} g({\pi}, \nu_{\omega}) \mathcal{k}}{1 - \gamma}\Big]
      \\ & = \mathbb{E}_{U_{0} \sim \xi_{U}} \Big[ \nabla_{\omega} V^{f_{\nu_{\omega}}(\pi)}_{U}(u)  \Big]
      \\ & - \sum_{u \in U, s\in\mathcal{S}}\frac{2 g({\pi}, \nu_{\omega})(u)}{\mathcal{k}  g({\pi}, \nu_{\omega}) \mathcal{k}} \Big\{   \nabla_{\omega}  \nu_{\omega}(s|u) \Big[ (P_{\mathcal{S}}^{\pi} \hat{V}_{\mathcal{S},\nu_{\omega}}^{\pi})(s) 
      \\ & - \frac{1}{\gamma}(\hat{V}_{\mathcal{S},\nu_{\omega}}^{\pi} - R_{\mathcal{S}}^{\pi})(s) \Big]  
      \\ & +
      \nu_{\omega}(s|u) \Big[ \nabla_{\omega} (P_{\mathcal{S}}^{\pi} \hat{V}_{\mathcal{S},\nu_{\omega}}^{\pi})(s) 
      - \frac{1}{\gamma} \nabla_{\omega}  \hat{V}_{\mathcal{S},\nu_{\omega}}^{\pi}(s) \Big]\Big\} \text{,}
  \end{split}
  \end{equation}
  where
  \begin{equation*}\label{HMASA_2_7}
    \begin{split}
      &\nabla_{\omega} V^{f_{\nu_{\omega}}(\pi)}_{U}(u) 
      \\ & = \mathbb{E}_{X_{t} \sim \eta(\cdot|f(\pi), u), S_{t} \sim \nu_{\omega}(\cdot|X_{t}), A_{t} \in \pi(\cdot|S_{t}) } \Big[   \nabla_{\omega} \ln \nu_{\omega}(S_{t}|X_{t})
      \\ &  \cdot 
      \big [r(S_{t},A_{t}) + \gamma \mathbb{E}_{Y \sim P_{2}(\cdot | S_{t}, A_{t})} [V^{f_{\nu_{\omega}}(\pi)}_{U}(Y)] \big] 
      \\ & + \gamma  \mathbb{E}_{ S_{t+1} \sim P_{\mathcal{S}\mathcal{A}}(\cdot | S_{t}, A_{t}), Y \sim \mu_{\omega}^{\dagger}(\cdot | S_{t+1})} [ \nabla_{\omega} \ln \mu_{\omega}^{\dagger}(Y|S_{t+1})
      \\ & \cdot V^{f_{\nu_{\omega}}(\pi)}_{U}(Y)]  \Big] \text{,}
  \end{split}
  \end{equation*}
  \begin{equation*}
    \begin{split}
      \nabla_{P_\nu} \mu^{\dagger}(s'|s) 
      & = \frac {\mathrm{Tr}({\partial x_{s'} x_s^{\top} P_{\nu}^{\top} (P_{\nu}P_{\nu}^{\top})^{-1}}  )}{ \partial P_\nu}
      \\ & =  \frac{\mathrm{Tr}( {\partial  (x_s x_{s'}^{\top} P_{\nu}) } (P_{\nu}P_{\nu}^{\top})^{-1} )}{ \partial P_\nu}
      \\ & + \frac{\mathrm{Tr}(x_{s'} x_s^{\top} P_{\nu}^{\top} {\partial  (P_{\nu}P_{\nu}^{\top})^{-1}} )}{ \partial P_\nu}
      \\ & = -  (P_{\nu}P_{\nu}^{\top})^{-1} x_{s} x_{s'}^{\top}  P_{\nu}^{\top} (P_{\nu}P_{\nu}^{\top})^{-1} 
      \\ & + x_{s'} x_s^{\top} (P_{\nu}P_{\nu}^{\top})^{-1} 
      \text{,}
  \end{split}
  \end{equation*}
  and $P_{2}(u' | s, a) = \sum_{s' \in \mathcal{S}} P_{\mathcal{S}\mathcal{A}}(s'|s,a) \mu^{\dagger}(u'|s') $.
\end{theorem}

\begin{proof}
  First, we analyze the gradient with respect to the parameter $\theta$. Theorem \ref{new_theorem4} has already provided the derivative of the performance function $J_{U} ({f_{\nu}(\pi_\theta)} )$, so we only consider the derivative with respect to $g(\pi_{\theta}, \nu)$. By the chain rule, the derivative of any vector $x$ with respect to $\mathcal{k}x\mathcal{k}$ is given by:
\begin{equation}\label{add_HMASA_3_1}
  \begin{split}
    \nabla_{x_i} \mathcal{k} x \mathcal{k}  & = \nabla_{x_i}  \sqrt{\sum_{j} x_{j}^{2}} = \frac{2 x_{i}}{\mathcal{k} x \mathcal{k}} 
    \text{.}
\end{split}
\end{equation}
Taking the derivative of $g(\pi, \nu )(s)$ with respect to the parameter $\theta$ yields:

\begin{equation}\label{add_HMASA_3_2}
  \begin{split}
    \nabla_{\theta}  g(\pi_{\theta}, \nu )(u) 
    & = \nabla_{\theta} \sum_{u \in U} \nu(s|u) \sum_{s' \in \mathcal{S}} P_{\mathcal{S}}^{\pi_{\theta}}(s'|s) \hat{V}_{\mathcal{S},\nu}^{\pi_{\theta}}(s') 
    \\ & - \nabla_{\theta} \big(P_{U}^{f_{\nu}(\pi_{\theta})} {V}_{U}^{f_{\nu}(\pi_{\theta})}\big) (s)
    \\ & = P_{\nu} P_{\mathcal{S}}^{\pi_{\theta}} \hat{V}_{\mathcal{S},\nu}^{\pi_{\theta}} - P_{U}^{f_{\nu}(\pi_{\theta})} {V}_{U}^{f_{\nu}(\pi_{\theta})}
    \\ & = \sum_{u \in U} \Big[ \nu(s|u) \sum_{s' \in \mathcal{S}} \big[ P_{\mathcal{S}}^{\pi_{\theta}}(s'|s) \nabla_{\theta} \hat{V}_{\mathcal{S},\nu}^{\pi_{\theta}}(s') 
    \\ & +  \sum_{a \in \mathcal{A}} P_{\mathcal{S}\mathcal{A}}(s'|s, a) \hat{V}_{\mathcal{S},\nu}^{\pi_{\theta}}(s') \nabla_{\theta} \pi_{\theta}(a|s) \big] \Big]
    \\ & - \sum_{u \in U}  \nu(s|u)  \frac{1}{\gamma} \nabla_{\theta} (\hat{V}_{\mathcal{S},\nu}^{\pi_{\theta}} - R_{\mathcal{S}}^{\pi_{\theta}})(s) 
    \text{.}
\end{split}
\end{equation}
where $P_{U}^{f_{\nu}(\pi_{\theta})} {V}_{U}^{f_{\nu}(\pi_{\theta})} = \frac{1}{\gamma} P_{\nu}(\hat{V}_{\mathcal{S},\nu}^{\pi_{\theta}} - R_{\mathcal{S}}^{\pi_{\theta}})$ follows from Equation \eqref{value_d_1} and ${V}_{U}^{f_{\nu}(\pi_{\theta})} = P_{\nu} \hat{V}_{\mathcal{S},\nu}^{\pi_{\theta}}$.
Substituting Equation \eqref{add_HMASA_3_2} into Equation \eqref{add_HMASA_3_1} yields:
\begin{equation}\label{add_HMASA_3_3}
  \begin{split}
    & \nabla_{\theta}  \frac{\mathcal{k} g({\pi_\theta}, \nu) \mathcal{k}}{1 - \gamma} 
    \\ & = \sum_{u \in U}\frac{2 g({\pi_\theta}, \nu)(u)}{\mathcal{k}  g({\pi_\theta}, \nu) \mathcal{k}} \Big[ \sum_{u \in U}  \nu(s|u) 
    \sum_{s' \in \mathcal{S}} \big[ P_{\mathcal{S}}^{\pi_{\theta}}(s'|s) \nabla_{\theta} \hat{V}_{\mathcal{S},\nu}^{\pi_{\theta}}(s') 
    \\ & +  \sum_{a \in \mathcal{A}} P_{\mathcal{S}\mathcal{A}}(s'|s, a) \hat{V}_{\mathcal{S},\nu}^{\pi_{\theta}}(s') \nabla_{\theta} \pi_{\theta}(a|s) \big] 
    \\ & - \sum_{u \in U}  \nu(s|u)  \frac{1}{\gamma} \nabla_{\theta} (\hat{V}_{\mathcal{S},\nu}^{\pi_{\theta}} - R_{\mathcal{S}}^{\pi_{\theta}})(s) \Big] 
    \text{.}
\end{split}
\end{equation}

Next, we turn to the analysis of the gradient with respect to the parameters $\omega$ of the encoding matrix. For brevity, we denote $\mu^{\dagger}(u' | s)$ as the element of matrix $P_{\nu}^{\dagger}$ located at the row corresponding to state $s \in \mathcal{S}$ and the column corresponding to abstract state $u' \in U$.
We differentiate the two terms $V^{f_{\nu_{\omega}}(\pi)}_{U}$ and $g({\pi}, \nu_{\omega})$ separately. First, the derivative of the first term is given by:

\begin{equation}\label{GEM_1_1}
  \begin{split}
    \nabla_{\omega} &V^{f_{\nu_{\omega}}(\pi)}_{U}(u) 
    \\ & = \nabla_{\omega} [R_{\mathcal{S}}^{\pi, \nu_{\omega}}(u) + \gamma \sum_{u'  \in U} P_{U}^{f_{\nu_{\omega}}(\pi)}(u'|u) V^{f_{\nu}(\pi)}_{U}(u')]
    \\ & =  \sum_{u \in U, s \in \mathcal{S}} \nu_{\omega}(s|u) \pi(a|s) r(s,a) \nabla_{\omega} \ln  \nu_{\omega}(s|u) 
    \\ & + \gamma \sum_{s \in \mathcal{S}, u' \in U}  \nu_{\omega}(s|u) P_{\mathcal{S}}^{\pi}(s'|s) \sum_{s' \in \mathcal{S}}  \nabla_{\omega} \mu_{\omega}^{\dagger}(u'|s') 
    \\ & \cdot \big[\big(\ln \nu(s|u) + \ln \mu_{\omega}^{\dagger}(u'|s')\big) V^{f_{\nu_{\omega}}(\pi)}_{U}(u')\big]
    \\ & + \gamma \sum_{u' \in U} P_{U}^{f_{\nu_{\omega}}(\pi)}(u'|u) \nabla_{\omega}  V^{f_{\nu_{\omega}}(\pi)}_{U}(u') \text{.}
\end{split}
\end{equation}

Let $P_{2}(u' | s, a) = \sum_{s' \in \mathcal{S}} P_{\mathcal{S}\mathcal{A}}(s'|s,a) \mu^{\dagger}(u'|s') $. Following this pattern, we obtain:
\begin{equation}\label{GEM_1_2}
  \begin{split}
    & =\mathbb{E}_{X_{t} \sim \eta(\cdot|f(\pi), u), S_{t} \sim \nu_{\omega}(\cdot|X_{t}), A_{t} \in \pi(\cdot|S_{t}) } \Big[   \nabla_{\omega} \ln \nu_{\omega}(S_{t}|X_{t})
    \\ &  \cdot 
    \big [r(S_{t},A_{t}) + \gamma \mathbb{E}_{Y \sim P_{2}(\cdot | S_{t}, A_{t})} [V^{f_{\nu_{\omega}}(\pi)}_{U}(Y)] \big] 
    \\ & + \gamma  \mathbb{E}_{ S_{t+1} \sim P_{\mathcal{S}\mathcal{A}}(\cdot | S_{t}, A_{t}), Y \sim \mu_{\omega}^{\dagger}(\cdot | S_{t+1})} [ \nabla_{\omega} \ln \mu_{\omega}^{\dagger}(Y|S_{t+1})
    \\ & \cdot V^{f_{\nu_{\omega}}(\pi)}_{U}(Y)]  \Big]  \text{,}
\end{split}
\end{equation}
where $ \eta(x|u, f(\pi)) = \sum_{t=0}^{\infty} \gamma^{t} P(X_t = x| X_0 = u, \pi)$.

Furthermore, the derivative of $\nabla_{\omega} \mu_{\omega}^{\dagger}(s'|s)$ can be expressed using matrix. Let $x_s \in \mathbb{R}^{|\mathcal{S}|}$ be a zero vector with the entry corresponding to state $s \in \mathcal{S}$ equal to $1$.
According to the rules of matrix, we have:
\begin{equation}\label{GEM_1_3}
  \begin{split}
     \mu_{\omega}^{\dagger}(s'|s) 
    & = x_s^{\top} P_{\nu_{\omega}}^{\dagger} x_{s'} = \mathrm{Tr}(x_{s'} x_s^{\top} P_{\nu_{\omega}}^{\dagger}  )
    \text{.}
\end{split}
\end{equation}

Based on the above expression, the derivative of $\mu_{\omega}^{\dagger}(s'|s) $ can be rewritten as:
\begin{equation}\label{GEM_1_4}
  \begin{split}
    \nabla_{P_\nu} \mu^{\dagger}(s'|s) 
    & = \nabla_{P_\nu} \mathrm{Tr}(x_{s'} x_s^{\top} P_{\nu}^{\top} (P_{\nu}P_{\nu}^{\top})^{-1}  )
    \\ & =  \frac {\mathrm{Tr}({\partial x_{s'} x_s^{\top} P_{\nu}^{\top} (P_{\nu}P_{\nu}^{\top})^{-1}}  )}{ \partial P_\nu}
    \text{,}
\end{split}
\end{equation}
where the final step follows from the differentiation rule for the trace of a matrix \cite{petersen2012matrix} (section 2, pp. 8, eq. 36).

By the chain rule for matrix calculus \cite{petersen2012matrix} (section 2.8.1, pp. 15, eq. 137), we have:
\begin{equation}\label{GEM_1_6}
  \begin{split}
    \nabla_{P_\nu} \mu^{\dagger}(s'|s) 
    & = \frac {\mathrm{Tr}({\partial x_{s'} x_s^{\top} P_{\nu}^{\top} (P_{\nu}P_{\nu}^{\top})^{-1}}  )}{ \partial P_\nu}
    \\ & =  \frac{\mathrm{Tr}( {\partial  (x_s x_{s'}^{\top} P_{\nu}) } (P_{\nu}P_{\nu}^{\top})^{-1} )}{ \partial P_\nu}
    \\ & + \frac{\mathrm{Tr}(x_{s'} x_s^{\top} P_{\nu}^{\top} {\partial  (P_{\nu}P_{\nu}^{\top})^{-1}} )}{ \partial P_\nu}
    \\ & = -  (P_{\nu}P_{\nu}^{\top})^{-1} x_{s} x_{s'}^{\top}  P_{\nu}^{\top} (P_{\nu}P_{\nu}^{\top})^{-1} 
    \\ & + x_{s'} x_s^{\top} (P_{\nu}P_{\nu}^{\top})^{-1} 
    \text{,}
\end{split}
\end{equation}
where the final equality follows from Lemma \ref{lemm2}.

For the second term, a derivation similar to that of Equation \eqref{add_HMASA_3_3} yields:
\begin{equation}\label{add_HMASA_3_4}
  \begin{split}
    & \nabla_{\omega}  \frac{\mathcal{k} g({\pi}, \nu_{\omega}) \mathcal{k}}{1 - \gamma} 
    \\ & = \sum_{u \in U, s\in\mathcal{S}}\frac{2 g({\pi}, \nu_{\omega})(u)}{\mathcal{k}  g({\pi}, \nu_{\omega}) \mathcal{k}} \nabla_{\omega}  \nu_{\omega}(s|u) \Big[ (P_{\mathcal{S}}^{\pi} \hat{V}_{\mathcal{S},\nu_{\omega}}^{\pi})(s) 
    \\ & - \nabla_{\omega} \frac{1}{\gamma}(\hat{V}_{\mathcal{S},\nu_{\omega}}^{\pi} - R_{\mathcal{S}}^{\pi})(s) \Big] 
    \\ & = \sum_{u \in U, s\in\mathcal{S}}\frac{2 g({\pi}, \nu_{\omega})(u)}{\mathcal{k}  g({\pi}, \nu_{\omega}) \mathcal{k}} \Big\{   \nabla_{\omega}  \nu_{\omega}(s|u) \Big[ (P_{\mathcal{S}}^{\pi} \hat{V}_{\mathcal{S},\nu_{\omega}}^{\pi})(s) 
    \\ & - \frac{1}{\gamma}(\hat{V}_{\mathcal{S},\nu_{\omega}}^{\pi} - R_{\mathcal{S}}^{\pi})(s) \Big]  
    \\ & +
    \nu_{\omega}(s|u) \Big[ \nabla_{\omega} (P_{\mathcal{S}}^{\pi} \hat{V}_{\mathcal{S},\nu_{\omega}}^{\pi})(s) 
    - \frac{1}{\gamma} \nabla_{\omega} \hat{V}_{\mathcal{S},\nu_{\omega}}^{\pi}(s) \Big]\Big\}
    \text{,}
\end{split}
\end{equation}
where the derivative of $\hat{V}_{\mathcal{S},\nu_{\omega}}^{\pi}(s)$ with respect to the parameters $\omega$ is given by:
\begin{equation}\label{add_HMASA_3_5}
  \begin{split}
    \nabla_{\omega} &  \hat{V}_{\mathcal{S},\nu_{\omega}}^{\pi}(s) 
    \\ & = \nabla_{\omega} R_{\mathcal{S}}^{\pi}(s) 
    \\ & + \nabla_{\omega} \sum_{s', \bar{s} \in \mathcal{S}, \bar{u}\in U} P_{\mathcal{S}}^{\pi}(\bar{s}|s) \nu_{\omega}^{\dagger}(\bar{u}|\bar{s}) \nu_{\omega}(s'|\bar{u}) \hat{V}_{\mathcal{S},\nu_{\omega}}^{\pi}(s') 
    \\ & =  \sum_{s', \bar{s} \in \mathcal{S}, \bar{u}\in U} P_{\mathcal{S}}^{\pi}(\bar{s}|s) \big[ \nabla_{\omega} \nu_{\omega}^{\dagger}(\bar{u}|\bar{s}) \nu_{\omega}(s'|\bar{u}) 
    \\ & +  \nu_{\omega}^{\dagger}(\bar{u}|\bar{s}) \nabla_{\omega} \nu_{\omega}(s'|\bar{u}) \big] \hat{V}_{\mathcal{S},\nu_{\omega}}^{\pi}(s')
    \\ & = \sum_{x \in \mathcal{S}} \eta(x| \pi, s) \sum_{\bar{s} \in \mathcal{S}, \bar{u}\in U} \big[ \nabla_{\omega} \nu_{\omega}^{\dagger}(\bar{u}|\bar{s}) \nu_{\omega}(x|\bar{u}) 
    \\ & +  \nu_{\omega}^{\dagger}(\bar{u}|\bar{s}) \nabla_{\omega} \nu_{\omega}(x|\bar{u}) \big] \hat{V}_{\mathcal{S},\nu_{\omega}}^{\pi}(x)
    \text{.}
\end{split}
\end{equation}

\end{proof}


\begin{algorithm}[tb]
  \caption{: Error-Based Homomorphic Policy Gradient Algorithm (EBHPG)}
  \label{alg_HM2}
  \begin{algorithmic}[2]
      \STATE Initial the policy ${{\theta}^{0}}$ and $\omega^{0}$
      \STATE $t = 0$
      \REPEAT
      \STATE $ C^{\pi_{{\theta}^{t}}} = P_{\mathcal{S}}^{\pi_{{\theta}^{t}}} P_{\nu_{\omega^{t}}}^{\dagger}$
      \STATE $V_{U}^{f_{\nu}(\pi_{{\theta}^{t}})} =  (I-\gamma P_{\nu_{\omega^{t}}} C^{\pi_{{\theta}^{t}}})^{-1} P_\nu R_{\mathcal{S}}^{\pi_{{\theta}^{t}}}$
      \STATE ${\theta}^{t+1} = {\theta}^{t} + lr * \frac{\partial}{\partial \theta}  \big(J_{U}(f_{\nu_{\omega^{t}}}(\pi_{\theta}^{t})) - g(\pi_{\theta}, \nu_{\omega^{t}}) \big)|_{\theta=\theta^t}$
      \STATE ${\omega}^{t+1} = {\omega}^{t} + lr * \frac{\partial}{\partial \omega}  \big(J_{U}(f_{\nu_{\omega^{t}}}(\pi_{\theta}^{t})) -  g(\pi_{\theta}, \nu_{\omega^{t}}) \big)|_{\omega=\omega^t}$
      \STATE $t = t + 1$
      \UNTIL{$\mathcal{k}V_{\mathcal{S}}^{\pi_{{\theta}^{t}}} - V_{\mathcal{S}}^{\pi_{\theta}^{t+1}}\mathcal{k} \leq \epsilon$}
  \end{algorithmic}
\end{algorithm}

At this point, we have derived the gradients of the performance function in Equation \eqref{polbt_1} with respect to the parameters $\theta$ and $\omega$. In summary, this subsection first introduces the objective function in Equation \eqref{polbt_1} and derives its corresponding policy gradient. Finally, we propose Algorithm \ref{alg_HM2}, which optimizes the lower bound of policy performance when the row space of $P_\nu$ does not contain $\text{span}(\mathcal{F})$.
Notably, the computational complexity for value function evaluation remains consistent with that of Algorithm \ref{alg_HM}. Table \ref{table1} summarizes the computational complexity of value function evaluation for the five algorithms.

\renewcommand{\arraystretch}{1.4}
\begin{table}[]
  \caption{The computational complexity of policy evaluation}
  \label{table1}
  \begin{center}
  \begin{small}
  \begin{tabular}{c | ccr}
  Method &  Computational Complexity  \\ \hline
  Policy Evaluation  & $\mathcal{O}( |\mathcal{S}||\mathcal{A}| + |\mathcal{S}|^{3} )$   \\ \hline
  HPG & $\mathcal{O}(|\mathcal{S}||\mathcal{A}| + |U||\mathcal{S}|^{2} + |U|^3)$ \\ \hline
  EBHPG & $\mathcal{O}(|\mathcal{S}||\mathcal{A}| + |U||\mathcal{S}|^{2} + |U|^3)$ \\ \hline
  \end{tabular}
  \end{small}
  \end{center}
\end{table}




\section{Numerical Results}\label{sec_4}
The numerical experiments aim to validate the theoretical framework and evaluate the performance of the proposed algorithms. First, we introduce the benchmark tasks used in the experiments. Next, we assess the performance of Algorithm \ref{alg_HM} under both conditions-when the sufficient condition is satisfied and when it is not. Finally, we compare Algorithm \ref{alg_HM2} against other methods on tasks with larger state spaces.
All experiments are conducted on the same hardware and use wall-clock time, the comparison directly reflects real-world efficiency (The experiments in this paper were conducted on a system equipped with an AMD Ryzen 7 5800X CPU and an NVIDIA GeForce RTX 3090 GPU).

\subsection{Experiments Setup}
Experiments are conducted on four representative tasks: Random Models \cite{Forghieri_Castel_Hyon_Pennec_2024}, Weakly coupled MDP \cite{sutton1999between}, Four-room gridworld (Example 5.2, p.110, \cite{sutton2018reinforcement}), and a tandem queue management problem inspired by real-world server systems \cite{ohno1987computing}.

\textbf{Random Models and Weakly Coupled MDP}:
We evaluate our slicing strategy on randomly generated MDP. For each $(s,a)$, transition probabilities $T(s,a,\cdot)$ are sampled randomly over $\mathcal{S}$, and rewards are drawn uniformly from $[0,1]$. A key variable is the transition matrix density-i.e., the proportion of nonzero entries-ranging from 10\% (sparse, nearly independent states) to 100\% (dense, smoother value functions). We construct weakly coupled MDP by partitioning the state space into disjoint clusters, each representing a local subtask or option. Transitions within each cluster are dense and randomly generated, while transitions across clusters are sparse, modeling loose dependencies between options. This design simulates hierarchical decision-making. Our slicing strategy effectively captures such structure, thereby improving value approximation and policy abstraction in hierarchical MDP.

\textbf{Four-room Gridworld}: The four-room gridworld consists of four rooms. The agent aims to reach a designated goal cell, with stochastic transitions: each action (North, South, East, West) succeeds with probability $0.8$ when the move is valid; otherwise, the agent remains in place. Upon reaching the goal, the agent is reset to the initial state, forming a continuing task. To assess scalability, we evaluate versions with increased state space sizes.

\textbf{Tandem Queue Management Problem}: The tandem queue management task involves two serial queues with parallel servers, where the agent adjusts server allocations to manage queue loads. Each queue allows three actions-add, retain, or remove a server-resulting in nine joint actions. The system's dimensionality is scalable via queue lengths and server capacities, following the design principles in \cite{9789809}.
 
\begin{figure}[!h]
  \centering \includegraphics[width=4.5in]{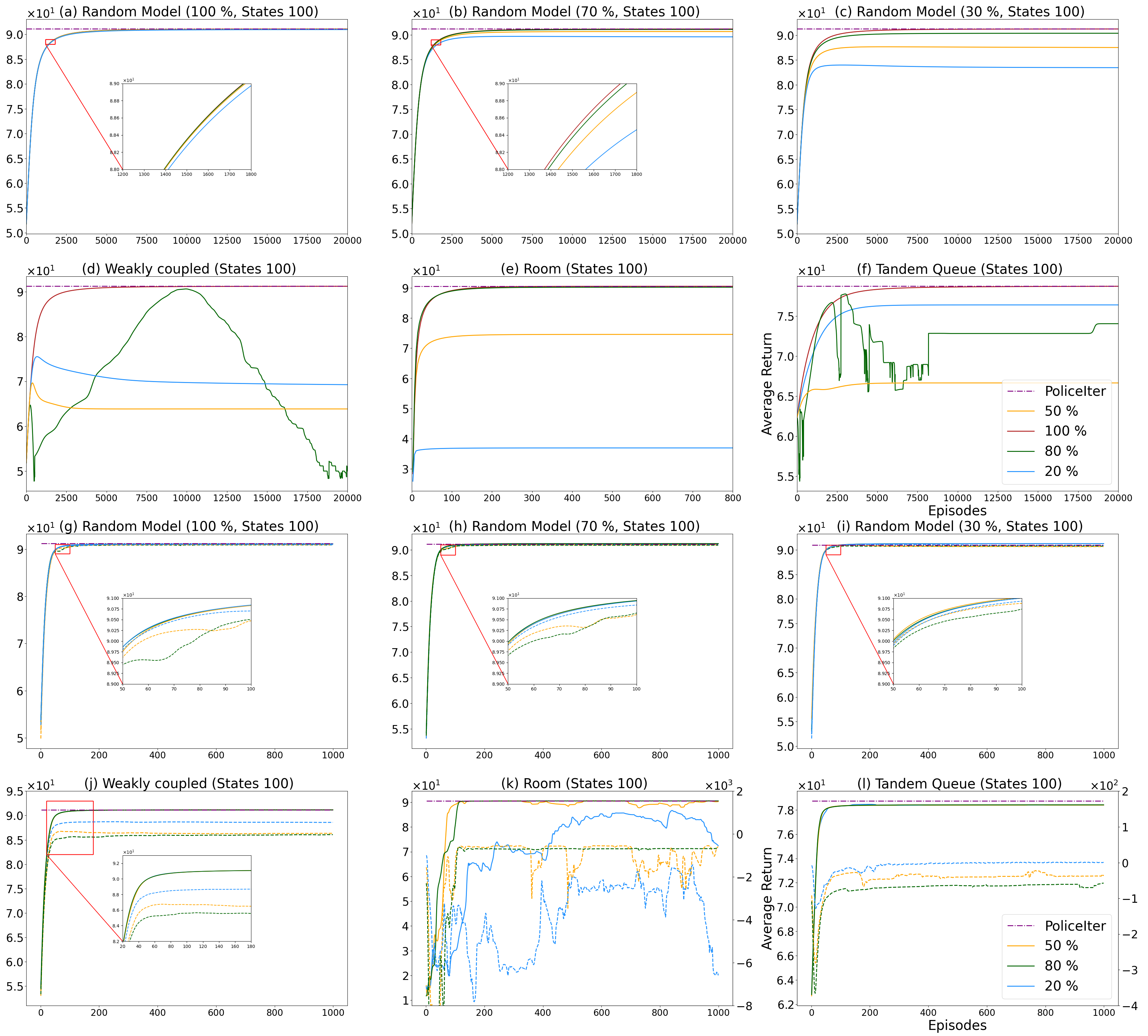}
  \caption{In the experimental results, the $x$-axis represents the number of iterations, while the $y$-axis indicates policy performance. At the top of each task subplot are the corresponding task names, with Task \texttt{"}Random Model\texttt{"} comprising three scenarios of different density levels (10\%, 50\%, and 100\%).
  The curves labeled \texttt{"}100\%\texttt{"}, \texttt{"}80\%\texttt{"}, \texttt{"}50\%\texttt{"}, and \texttt{"}20\%\texttt{"} in the figure correspond to different settings of the abstract state space size, where $|U| = int(0.2 * r)$, $|U| = int(0.5 * r)$, $|U| = int(0.8 * r)$, and $|U| = int(r)$, respectively.
  Figures (a)-(f) show the results of Algorithm \ref{alg_HM} under different values of $|U|$, while Figures (g)-(l) present the results of Algorithm \ref{alg_HM2} under the same settings. In all figures, the purple dashed line represents the policy performance after $40,000$ iterations of the policy iteration algorithm, which serves as an approximation of the optimal policy performance. In Figures (g)-(l), solid lines indicate actual policy performance (correspond to the left $y$-axis), while dashed lines represent the performance lower bound (In subfigures (k) and (l), the dashed lines correspond to the right $y$-axis.), corresponding to the $J_{U}(f_{\nu}(\tilde{\pi})) - \frac{\mathcal{k}g(\tilde{\pi}, \nu ) \mathcal{k}}{ 1 - \gamma }$ term in Equation \eqref{polbt_1}.}
  \label{result_1}
\end{figure}

\subsection{Experiments for theoretical validation}

To validate the theoretical results, we evaluate Algorithm \ref{alg_HM} under two settings: one where optimal policy equivalence holds ($|U| = r$) and one where it does not ($|U| < r$), where $r$ is defined in Definition \ref{de1_3}. In the experiments, all tasks are set with $|\mathcal{S}| = 100$. Specifically, random model task and weakly coupled MDP uses $|\mathcal{A}| = 10$, the four-room gridworld uses $|\mathcal{A}| = 4$, and the tandem queue management task uses $|\mathcal{A}| = 3$. 
It is worth noting that for the random models, we evaluated cases with transition matrix densities of 10\%, 50\%, and 100\%.
To eliminate errors due to value function approximation, Algorithm \ref{alg_HM} computes value functions using planning.

The experimental results are shown in Figure \ref{result_1}. To verify the correctness of Theorems \ref{theorem2} and \ref{theorem4}, we compare the performance of Algorithm \ref{alg_HM} under both conditions: when the sufficient condition is satisfied (represented by the curve labeled 100\%) and when it is not (represented by the curves labeled 80\%, 50\%, and 20\%). Here, a label such as 80\% indicates that $|U| = int(0.8 * r)$, and the others follow accordingly. Each task is solved using the standard policy iteration (PolicyIter) algorithm \cite{sutton2018reinforcement} for 2000 iterations to approximate the optimal solution, shown as a dashed line in the figure.

According to the results of Theorem \ref{theorem2} and Theorem \ref{theorem4}, optimal policy equivalence holds when the row space of $P_\nu$ contains $\text{span}(\mathcal{F})$. As shown in Figure \ref{result_1} (a)-(e), the curves labeled "100 \%" correspond to the cases where Algorithm \ref{alg_HM} satisfies the sufficient condition. In these settings, the policy consistently converges to the optimal value across all tasks. Moreover, the monotonic improvement in policy performance provides empirical support for the correctness of the policy gradient (Theorem \ref{new_theorem4}).
In contrast, when the sufficient condition is not satisfied (i.e., the curves labeled "80\%", "50\%", and "20\%"), the algorithm does not necessarily converge to the optimal solution and exhibits noticeable oscillations. This highlights a key limitation: satisfying the sufficient condition becomes computationally expensive when the rank $r$ is large. Therefore, the development of Theorem \ref{POLB} is essential.

According to Theorem \ref{POLB}, Algorithm \ref{alg_HM2} optimizes a lower bound on policy performance. Experimental results are presented in Figure \ref{result_1} (g)-(l), where dashed lines represent the performance lower bound (Equation \eqref{polbt_1}; $J_{U}(f_{\nu}(\tilde{\pi})) - \frac{\mathcal{k}g(\tilde{\pi}, \nu ) \mathcal{k}}{ 1 - \gamma}$), and solid lines represent actual policy performance (Equation \eqref{polbt_1}; $J_{\mathcal{S}}(\tilde{\pi})$). The results show that as the lower bound improves, the actual performance also increases accordingly. This observation further supports the validity of the policy gradient proposed in Theorem \ref{theorem5}.
From the perspective of the objective function in Equation \ref{polbt_1}, the term $\frac{\mathcal{k} g(\tilde{\pi}, \nu ) \mathcal{k}}{ 1 - \gamma }$ acts as a penalty, with the penalty factor proportional to $\frac{1}{ 1 - \gamma }$. Consequently, the oscillatory behavior observed in some experiments is expected, as a large penalty factor may lead to large gradients and thus unstable updates.

\begin{figure*}[t]
  \centering \includegraphics[width=5in]{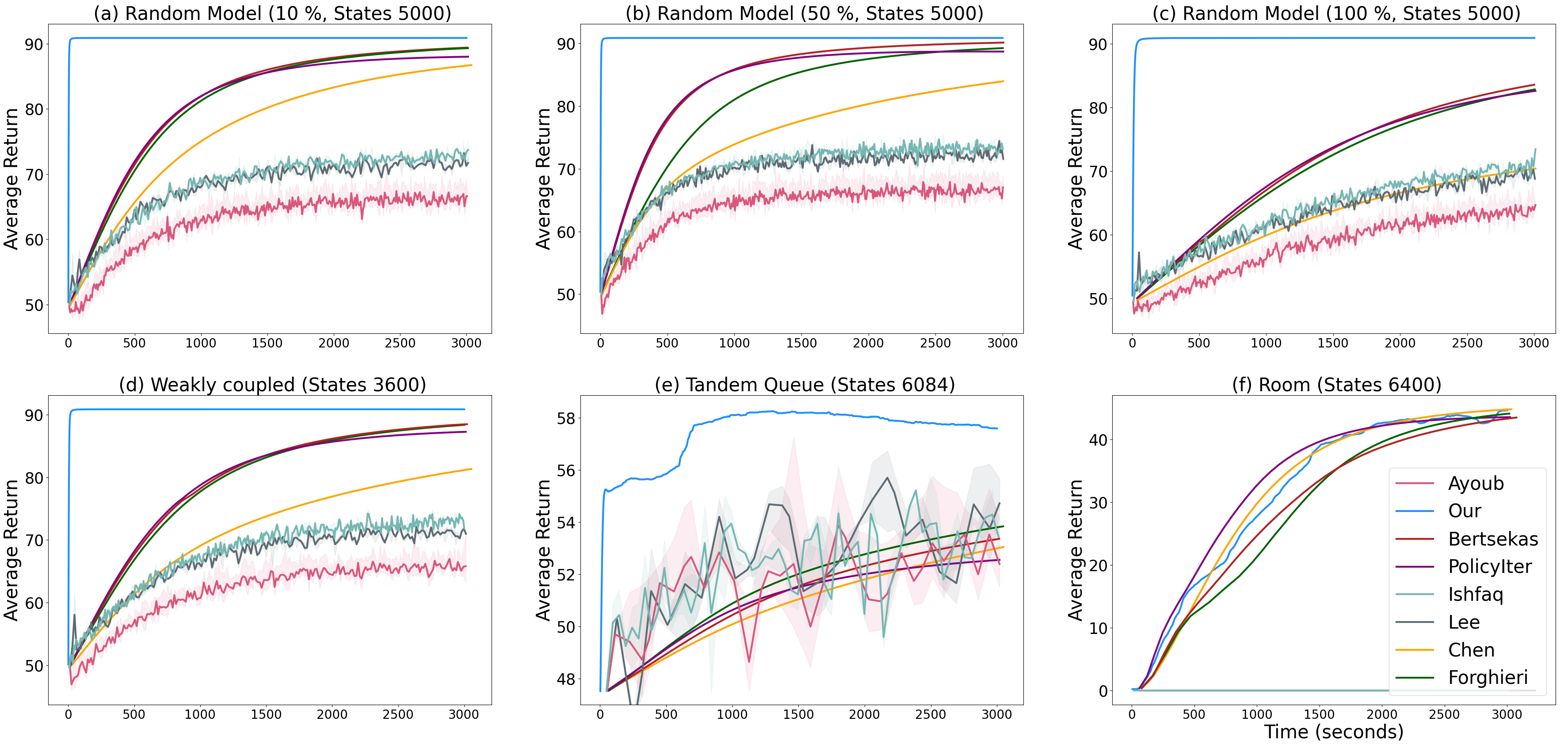}  
  \caption{In the experimental results, the x-axis represents wall-clock time (Execution time on a physical computing system), while the y-axis indicates policy performance.  In the experiments, the results corresponding to Algorithm \ref{alg_HM2} are labeled as \texttt{"}Our\texttt{"}. Accordingly, in the figure, the solid line represents the average over five runs, while the shaded region indicates the range between the maximum and minimum values.
  }\label{result_2}
\end{figure*}

\subsection{Algorithm performance evaluation}

In the previous subsection, we validated the sufficient condition for optimal policy equivalence on simple tasks. In this subsection, we consider more complex tasks, for which Algorithm \ref{alg_HM} is no longer suitable. This is because satisfying the sufficient condition typically requires $|U| > r \approx |\mathcal{S}|$; for instance, in the Four-Room task, $r = |\mathcal{S}|$, implying that the computational complexity of Algorithm \ref{alg_HM} becomes comparable to that of standard policy iteration.

To address this limitation, we leverage Algorithm \ref{alg_HM2} to demonstrate the advantage of homomorphic mappings in large state space. In the experiments,  random model task has $|\mathcal{S}| = 5000$ and $|\mathcal{A}| = 10$; weakly coupled MDP has $|\mathcal{S}| = 3600$ and $|\mathcal{A}| = 10$; the four-room gridworld task has $|\mathcal{S}| = 6400$ and $|\mathcal{A}| = 4$; and the tandem queue management task has $|\mathcal{S}| = 6084$ and $|\mathcal{A}| = 3$. To eliminate value function approximation error, all model-based methods compute the value function using exact planning.

The primary baseline is the standard Policy Iteration (PolicyIter) algorithm. Comparative methods include a classical aggregation technique (Bertsekas) \cite{24227}, as well as five recent approaches proposed by Ayoub et al. \cite{pmlr-v119-ayoub20a}, Chen \cite{chen2022adaptive}, Forghieri et al. \cite{Forghieri_Castel_Hyon_Pennec_2024}, Ishfaq et al. \cite{ishfaq21a}, and Lee et al. \cite{DBLP}. It is worth noting that the methods by Ayoub, Ishfaq, and Lee are model-free, whereas the remaining approaches are model-based. 
To ensure fairness and reproducibility, each algorithm is executed using its default hyperparameter configuration.
It is worth noting that, to ensure fairness and consistency across algorithms, policy optimization is implemented via policy gradient (The learning rate is set to $1 \times 10^{-3}$).  In addition, our method uses $|U| = int(0.5 * r)$ in all experiments. The hyperparameters of the baseline methods are set to their default values.

The experimental results are shown in Figure \ref{result_2}, where each curve depicts the evolution of policy performance over wall-clock time.
From the experimental results, it is evident that Algorithm \ref{alg_HM2} (labeled as \texttt{"}Ours\texttt{"}) consistently outperforms other methods across all tasks except the Four-Room environment, where performance is comparable. First, model-based methods generally outperform model-free approaches, as they compute value functions via exact planning rather than sampling. For model-based state aggregation methods, our approach performs the aggregation entirely through matrix operations (as formalized in Theorem \ref{theorem4}), whereas the baseline methods rely on value-based procedures that involve complex computations, typically implemented using nested for-loops. Since matrix operations are substantially more efficient than iterative loops in practical computation, especially in large state spaces, our method demonstrates superior computational efficiency.

In subfigure (f) of Figure \ref{result_2}, we observe that nearly all methods fail to surpass the baseline \texttt{"}PolicyIter\texttt{"}. A possible explanation is the extremely sparse reward structure in the Four Room environment. For model-free methods, the combination of a large state space and sparse rewards makes it difficult to explore critical states effectively. For model-based methods, the sparsity of the reward function may slow down policy iteration, thereby reducing the efficiency gains from aggregation. A rigorous theoretical analysis of this phenomenon is left for future work.

\section{Conclusion}\label{sec_5}

This work presents a  framework for state aggregation in Markov Decision Processes through the lens of homomorphic mappings. By relaxing the stringent constraints of classical homomorphic MDPs, we introduce the notion of homomorphic Markov chains, enabling value-function linearity to be enforced over individual policy-induced Markov chains rather than the entire policy space. Under this relaxed framework, we derive sufficient conditions for optimal policy equivalence, thereby ensuring that policies optimized in the abstract space remain optimal in the ground MDP.

In cases where these sufficient conditions are not met, we analyze the resulting approximation error and establish theoretical bounds on policy performance degradation. These insights motivate the development of two algorithms: Homomorphic Policy Gradient (HPG), which enforces exact homomorphism, and Error-Bounded Homomorphic Policy Gradient (EBHPG), which balances approximation accuracy with computational efficiency via least-squares projections.
Experimental results across diverse benchmark environments validate the effectiveness of our methods. In particular, we demonstrate that the proposed algorithms achieve consistent performance improvements over existing state aggregation techniques, both in idealized and approximate settings. These findings highlight the practical viability of homomorphism-based abstraction for efficient and reliable reinforcement learning in large-scale or structured decision processes.

Regarding the limitations of this work, we first note that the sufficient condition for optimal policy equivalence may still be overly restrictive. In scenarios involving approximation errors, our method may fail to guarantee convergence to the optimal policy, which could limit the algorithm's performance. Moreover, our analysis does not extend to continuous state spaces, presenting a potential direction for future research.

\section*{references}

\bibliographystyle{IEEEtran}
\bibliography{IEEEexample}

\end{document}